\documentclass{article} 
\usepackage{iclr2025_conference,times}

\usepackage{amsmath,amsfonts,bm}
\newcommand{\E}{\mathbb{E}}
\newcommand{\R}{\mathbb{R}}
\def\1{\bm{1}}

\usepackage{hyperref}
\usepackage[capitalize,noabbrev]{cleveref}
\usepackage{url}

\usepackage{algorithm}
\usepackage[noend]{algorithmic}
\usepackage{amsmath}
\usepackage{amssymb}
\usepackage{mathtools}
\usepackage{mathrsfs}
\usepackage{amsthm}
\usepackage{comment}
\usepackage{enumitem}
\usepackage{subcaption}
\usepackage{nicefrac}
\usepackage{tikz}
\usetikzlibrary{arrows.meta,backgrounds,patterns,positioning,shapes.geometric,quotes}
\def\est{\mathrm{est}}

\usepackage{tikz}
\usetikzlibrary{automata,positioning}
\usepackage{tikz-cd}
\def\E{\mathbb{E}}
\def\N{\mathbb{N}}
\newtheorem{theorem}{Theorem}[section]
\newtheorem{definition}[theorem]{Definition}
\newtheorem{lemma}[theorem]{Lemma}
\newtheorem{example}[theorem]{Example}
\newtheorem{remark}[theorem]{Remark}
\newtheorem{assumption}{Assumption}[section]
\usepackage{bbm}
\newtheorem{corollary}[theorem]{Corollary}
\newtheorem{discussion}[theorem]{Discussion}

\title{Efficient Reinforcement Learning for Global Decision Making in the Presence of Local Agents at Scale}

\author{Emile Anand\thanks{Work done while author was a visiting student at Carnegie Mellon University.} \\
Computing and Mathematical Sciences\\ California Institute of Technology\\
\texttt{eanand@caltech.edu} \\
\And
Guannan Qu \\
Department of Electrical and Computer Engineering \\
Carnegie Mellon University \\
\texttt{gqu@andrew.cmu.edu}
}

\iclrfinalcopy 
\begin{document}

\maketitle

\begin{abstract}
    We study reinforcement learning for global decision-making in the presence of local agents, where the global decision-maker makes decisions affecting all local agents, and the objective is to learn a policy that maximizes the joint rewards of all the agents. Such problems find many applications, e.g. demand response, EV charging, queueing, etc. In this setting, scalability has been a long-standing challenge due to the size of the state space which can be exponential in the number of agents. This work proposes the \texttt{SUBSAMPLE-Q} algorithm where the global agent subsamples $k\leq n$ local agents to compute a policy in time that is polynomial in $k$. We show that this learned policy converges to the optimal policy in the order of $\tilde{O}(1/\sqrt{k}+{\epsilon}_{k,m})$ as the number of sub-sampled agents $k$ increases, where ${\epsilon}_{k,m}$ is the Bellman noise. Finally, we validate the theory through numerical simulations in a demand-response setting and a queueing setting.
\end{abstract}

\section{Introduction}
Global decision-making for local agents, where a global agent makes decisions that affect a large number of local agents, is a classical problem that has been widely studied in many forms \citep{foster2022on,qin2023learning,foster2023modelfree} and can be found in many applications, e.g. network optimization, power management, and electric vehicle charging \citep{7438918, doi:10.1177/0278364915581863, 7990560}. However, a critical challenge is the uncertain nature of the underlying system, which can be very hard to model precisely. Reinforcement Learning (RL) has seen an impressive performance in a wide array of applications, such as the game of Go \citep{Silver_Huang_Maddison_Guez_Sifre_van_den_Driessche_Schrittwieser_Antonoglou_Panneershelvam_Lanctot_et_al._2016}, autonomous driving \citep{9351818}, and robotics \citep{doi:10.1177/0278364913495721}. More recently, RL has emerged as a powerful tool for learning to control unknown systems \citep{ghai2023online, NEURIPS2023_a7a7180f,lin2024online,pmlr-v247-lin24a}, and hence has great potential for decision-making for multi-agent systems, including the problem of global decision making for local agents.

However, RL for multi-agent systems, where the number of agents increases, is intractable due to the curse of dimensionality \citep{Blondel_Tsitsiklis_2000}. For instance, RL algorithms such as tabular $Q$-learning and temporal difference (TD) learning require storing a $Q$-function \citep{10.5555/560669, 10.5555/1324761} that is as large as the state-action space. However, even if the individual agents' state space is small, the global state space can take values from a set of size exponentially large in the number of agents. In the case where the system's rewards are not discounted, reinforcement learning on multi-agent systems is provably NP-hard \citep{Blondel_Tsitsiklis_2000},
This problem of scalability has been observed in a variety of settings \citep{b0e74184-2114-3e45-b092-dfbc8fefcf91,10.5555/1622434.1622447}. A promising line of research that has emerged over recent years constrains the problem to a networked instance to enforce local interactions between agents \citep{DBLP:journals/corr/abs-2006-06555,DBLP:conf/nips/LinQHW21,pmlr-v120-qu20a,9867152,Chu2020Multi-agent}. This has led to scalable algorithms where each agent only needs to consider the agents in its neighborhood to derive approximately optimal solutions. However, these results do not apply to our setting where one global agent interacts with many local agents. This
can be viewed as a star graph, where the neighborhood of the central decision-making agent is large.

Beyond the networked formulation, another exciting line of work that addresses this intractability is mean-field RL \citep{pmlr-v80-yang18d}. The mean-field RL approach assumes that all the agents are homogeneous in their state and action spaces, which allows the interactions between agents to be approximated by a representative ``mean'' agent. This reduces the complexity of $Q$-learning to polynomial in the number of agents, and learns an approximately optimal policy where the approximation error decays with the number of agents \citep{doi:10.1137/20M1360700,gu2022dynamic}. However, mean-field RL does not directly transfer to our setting as the global decision-making agent is heterogeneous to the local agents. Further, when the number of local agents is large, it might still be impractical to store a polynomially-large $Q$-table (where the polynomial's degree is the size of the state space for a single agent). This motivates the following fundamental question: \emph{can we design a fast and competitive policy-learning algorithm for a global decision-making agent in a system with many local agents?}

\textbf{Contributions.} We answer this question affirmatively.  Our key contributions are outlined below.

\begin{itemize}[leftmargin=*]
    \item \textbf{Subsampling Algorithm. } We propose \texttt{SUBSAMPLE-Q}, an algorithm designed to address the challenge of global decision-making in systems with a large number of pseudo-heterogeneous local agents. We model the problem as a Markov Decision Process with a global decision-making agent and $n$ local agents. \texttt{SUBSAMPLE-Q} (\cref{algorithm: approx-dense-tolerable-Q-learning,algorithm: approx-dense-tolerable-Q*}) first chooses $k$ local agents to learn a deterministic policy $\hat{\pi}_{k,m}^{\est}$, where $m$ is the number of samples used to update the $Q$-function's estimates, by performing mean-field value iteration on the $k$ local agents to learn ${Q}^\est_{k,m}$, which can be viewed as a smaller $Q$ function of size polynomial in $k$, instead of polynomial in $n$ (as done in the mean-field RL literature). It then deploys a stochastic policy $\pi_{k,m}^\est$ that chooses $k$ local agents, uniformly at random, at each step to find an action for the global agent using ${\hat{\pi}}_{k,m}^\est$.    
    \item \textbf{Theoretical Guarantee. }\cref{theorem: performance_difference_lemma_applied} shows that the performance gap between ${{\pi}_{k,m}^\est}$ and the optimal policy ${ \pi^*}$ is ${{O}(\!\frac{1}{\sqrt{k}}\!+\!\epsilon_{k,m}\!)}$, where $\epsilon_{k,m}$ is the Bellman noise in ${{Q}_{k,m}^\est}$. The choice of $k$ reveals a fundamental trade-off between the size of the $Q$-table stored and the optimality of ${\pi}_{k,m}^\est$. For $k\!=\!O(\log n)$, \texttt{SUBSAMPLE-Q} runs in time polylogarithmic in $n$, creating an exponential speedup from the previously best-known polytime mean-field RL methods, with a decaying optimality gap.
    \item \textbf{Numerical Simulations.} We demonstrate the effectiveness of \texttt{SUBSAMPLE-Q} in a power system demand-response problem in Example \ref{example: demand-response}, and in a queueing problem in Example \ref{example: queueing}. A key inspiration of our approach is the power-of-two-choices in the queueing theory literature \citep{10.5555/924815}, where a dispatcher subsamples two queues to make decisions. Our work generalizes this to a broader decision-making problem.
\end{itemize}
While our result is theoretical in nature, it is our hope that \texttt{SUBSAMPLE-Q} will lead to further investigation into the power of sampling in Markov games and inspire practical algorithms.

\section{Preliminaries}
\label{section: preliminaries}
\textbf{Notation.} For $k,m\!\in\!\N$ where $k\!\leq\!m$, let $\binom{[m]}{k}$ denote the set of $k$-sized subsets of $[m]\!=\!\{1,\dots,m\}$.
Let $\overline{[m]}=\{0\}\cup[m]$. For any vector $z\in\mathbb{R}^d$, let $\|z\|_1$ and $\|z\|_\infty$ denote the standard $\ell_1$ and $\ell_\infty$ norms of $z$ respectively.
Let $\|\mathbf{A}\|_1$ denote the matrix $\ell_1$-norm of $\mathbf{A}\!\!\in\!\!\mathbb{R}^{n\times m}$. Given a collection of variables $s_1,\!\dots,\!s_n$ the shorthand $s_{\Delta}$ denotes the set $\{s_i\!\!:i\!\in\!\Delta\}$ for $\Delta\!\!\subseteq\!\![n]$. We use $\tilde{O}(\cdot)$ to suppress polylogarithmic factors in all problem parameters except $n$. For any discrete measurable space $(\mathcal{S}, \mathcal{F})$, the total variation distance between probability measures $\mu_1, \mu_2$ is given by $\mathrm{TV}(\mu_1,\mu_2)\!=\!\frac{1}{2}\sum_{s\in\mathcal{S}}|\mu_1(s)-\mu_2(s)|$. Finally, for $C\subset\mathbb{R}$, $\Pi^C:\mathbb{R}\to C$ denotes a projection onto $C$ in $\ell_1$-norm.

\subsection{Problem Formulation}

\textbf{Problem Statement.} We consider a system of $n+1$ agents given by $\mathcal{N} = \{0\}\cup [n]$. Let agent $0$ be the ``global agent'' decision-maker, and agents $[n]$ be the ``local'' agents. In this model, each agent $i\in [n]$ is associated with a state $s_i \in \mathcal{S}_l$, where $\mathcal{S}_l$ is the local agent's state space. The global agent is associated with a state $s_g \in \mathcal{S}_g$ and action $a_g \in \mathcal{A}_g$, where $\mathcal{S}_g$ is the global agent's state space and $\mathcal{A}_g$ is the global agent's action space. The global state of all agents is given by $(s_g,s_1,\dots,s_n)\in \mathcal{S}:=\mathcal{S}_g\times\mathcal{S}_l^n$. At each time-step $t$, the next state for all the agents is independently generated by stochastic transition kernels $P_g:\mathcal{S}_g\times\mathcal{S}_g\times\mathcal{A}_g\to[0,1]$ and $P_l:\mathcal{S}_l\times\mathcal{S}_l\times\mathcal{S}_g\to [0,1]$ as follows:
\begin{equation}\label{equation: global_transition}s_g(t+1) \sim P_g(\cdot|s_g(t), a_g(t)),\end{equation}
\begin{equation}\label{equation: local_transition}s_i(t+1) \sim P_l(\cdot|s_i(t), s_g(t)), \forall i\in [n]\end{equation}
The global agent selects $a_g(t)\in\mathcal{A}_g$. Next, the agents receive a structured reward $r:\mathcal{S}\times\mathcal{A}_g\to\mathbb{R}$, given by \cref{equation: rewards}, where the choice of functions $r_g$ and $r_l$ is flexible and application-specific.
\begin{equation}\label{equation: rewards}r(s, a_g)=\underbrace{r_g(s_g, a_g)}_{\text{global component}}+\frac{1}{n}\sum_{i\in [n]}\underbrace{r_l(s_i,s_g)}_{\text{local component}}\end{equation}

We define a policy $\pi:\mathcal{S}\to\mathcal{P}(\mathcal{A}_g)$ as a map from states to distributions of actions such that $a_g\sim\pi(\cdot|s)$. When a policy is executed, it generates a trajectory $(s^0, a^0_g, r^0),\dots,(s^T,a_g^T,r^T)$ via the process $a_g^{t}\sim \pi(s^t), s^{t+1} \sim (P_g, P_l)(s^t,a_g^t)$, initialized at $s^1\sim d_0$. We write $\mathbb{P}^\pi[\cdot]$ and $\E^\pi[\cdot]$ to denote the law and corresponding expectation for the trajectory under this process.
The goal of the problem is to then learn a policy $\pi$ that maximizes the value function $V:\pi\times\mathcal{S}\to\mathbb{R}$ which is the expected discounted reward for each $s\in\mathcal{S}$ given by $V^\pi(s) = \mathbb{E}^{\pi}[\sum_{t=0}^\infty \gamma^t r(s(t),a_g(t))|s(0)=s]$,
where $\gamma\in(0,1)$ is a discounting factor. We write $\pi^*$ as the optimal deterministic policy, which maximizes $V^{\pi}(s)$ at all states. This model characterizes a crucial decision-making process in the presence of multiple agents where the information of all local agents is concentrated towards the decision maker, the global agent. So, the goal of the problem is to learn an approximately optimal policy which jointly minimizes the sample and computational complexities of learning the policy.

We make the following standard assumptions:

\begin{assumption}[Finite state/action spaces]
    \emph{We assume that the state spaces of all the agents and the action space of the global agent are finite:  $|\mathcal{S}_l|, |\mathcal{S}_g|,|\mathcal{A}_g|<\infty$.}
\end{assumption}

\begin{assumption}[Bounded rewards]\label{assumption: bounded rewards}
    \emph{The global and local components of the reward function are bounded. Specifically, $\|r_g(\cdot,\cdot)\|_\infty \leq \tilde{r}_g$, and
    $\|r_l(\cdot,\cdot)\|_\infty \leq \tilde{r}_l$. Then, $\|r(\cdot,\cdot)\|_\infty \leq \tilde{r}_g + \tilde{r}_l := \tilde{r}$.}
\end{assumption}

\begin{definition}[$\epsilon$-optimal policy] \emph{Given a policy simplex $\Pi$, a policy $\pi\in\Pi$ is $\epsilon$-optimal if for all $s\in\mathcal{S}$, $V^\pi(s) \geq \sup_{\pi^*\in\Pi} V^{\pi^*}(s) - \epsilon$.}
\end{definition}

\begin{remark}\emph{
While this model requires the $n$ local agents to have homogeneous transition and reward functions, it allows heterogeneous initial states, which captures a pseudo-heterogeneous setting. For this, we assign a \emph{type} to each local agent by letting $\mathcal{S}_l = \mathcal{Z}\times\bar{\mathcal{S}_l}$, where $\mathcal{Z}$ is a set of different types for each local agent, which is treated as part of the state for each local agent. This type state will be heterogeneous and will remain unchanged throughout the transitions. Hence, the transition and reward function will be different for different types of agents. Further, by letting $s_g\in\mathcal{S}_g:=\prod_{z\in\mathcal{Z}}[\bar{S}_g]_z$ and $a_g\in \mathcal{A}_g:=\prod_{z\in\mathcal{Z}} [\bar{A}_g]_z$ correspond to a state/action vector where each element corresponds to a type $z\in\mathcal{Z}$, the global agent can uniquely signal agents of each type.}
\end{remark}

\subsection{Related Work}This paper relates to two major lines of work which we describe below.

\textit{Multi-agent RL (MARL)}. MARL has a rich history starting with early works on Markov games used to characterize the decision-making process \citep{doi:10.1073/pnas.39.10.1095, littman}, which can be regarded as a multi-agent extension to the Markov Decision Process (MDP). MARL has since been actively studied \citep{zhang2021multiagent} in a broad range of settings, such as cooperative and competitive agents. MARL is most similar to the category of ``succinctly described'' MDPs \citep{Blondel_Tsitsiklis_2000} where the state/action space is a product space formed by the individual state/action spaces of multiple agents, and
where the agents interact to maximize an objective function. Our work, which can be viewed as an essential stepping stone to MARL, also shares the curse of dimensionality. 

A line of celebrated works \citep{pmlr-v120-qu20a,Chu2020Multi-agent,DBLP:journals/corr/abs-2006-06555,DBLP:conf/nips/LinQHW21,9867152} constrain the problem to networked instances to enforce local agent interactions and find policies that maximize the objective function which is the expected cumulative discounted reward. By exploiting Gamarnik's spatial exponential decay property from combinatorial optimization \citep{gamarnik2009correlation}, they overcome the curse of dimensionality by truncating the problem to only searching over the policy space derived from the local neighborhood of agents that are atmost $\kappa$ away from each other to find an $O(\rho^{k+1})$ approximation of the maximized objective function for $\rho\in(0,1)$. However, since their algorithms have a complexity that is exponential in the size of the neighborhood, they are only tractable for sparse graphs. Therefore, these algorithms do not apply to our decision-making problem which can be viewed as a dense star graph (see \cref{appendix: math_background}). The recently popular work on V-learning \citep{jin2021vlearning} reduces the dependence of the product action space to an additive dependence. However, since our work focuses on the action of the global decision-maker, the complexity in the action space is already minimal. Instead, our work focuses on reducing the complexity of the joint state space which has not been generally accomplished for dense networks.

\textit{Mean-Field RL}. Under assumptions of homogeneity in the state/action spaces of the agents, the problem of densely networked multi-agent RL was answered in \cite{pmlr-v80-yang18d,doi:10.1137/20M1360700,gu2022dynamic, gu2022meanfield,subramanian2022decentralized}  which approximates the learning problem with a mean-field control approach where the approximation error scales in $O(1/\sqrt{n})$. To overcome the problem of designing algorithms on probability measure spaces, they study MARL under Pareto optimality and use the (functional) strong law of large numbers to consider a lifted state/action space with a representative agent where the rewards and dynamics of the system are aggregated. \cite{cui2022learning,Hu_Wei_Yan_Zhang_2023,10.1214/23-AAP1949} introduce heterogeneity to the mean-field approach using graphon mean-field games; however, there is a loss in topological information when using graphons to approximate finite graphs, as graphons correspond to infinitely large adjacency matrices. Additionally, graphon mean-field RL imposes a critical assumption of the existence of graphon sequences that converge in cut-norm to the problem instance. Another mean-field RL approach that partially introduces heterogeneity is in a line of work considering major and minor agents. This has been well studied in the competitive setting \citep{carmona2014probabilistic,carmona2016finite}. In the cooperative setting, \cite{10.5555/3586589.3586718,cui2023multiagent} are most related to our work, which collectively consider a setting with $k$ classes of homogeneous agents, but their mean-field analytic approaches does not converge to the optimal policy upon introducing a global decision-making agent. Typically, these works require Lipschitz continuity assumptions on the reward functions which we relax in our work. Finally, the algorithms underlying mean-field RL have a runtime that is polynomial in $n$, whereas our \texttt{SUBSAMPLE-Q} algorithm has a runtime that is polynomial in $k$.

\textit{Other Related Works}. A line of works have similarly exploited the star-shaped network in cooperative multi-agent systems. \cite{pmlr-v202-min23a,chaudhari2024peertopeerlearningdynamicswide} studied the communication complexity and mixing times of various learning settings with purely homogeneous agents, and \cite{do2023multiagent} studied the setting of heterogeneous linear contextual bandits to yield a no-regret guarantee. We extend this work to the more challenging setting in reinforcement learning.

\subsection{Technical Background}
\paragraph{Q-learning.}\label{RL review}
To provide background for the analysis in this paper, we review a few key technical concepts in RL. At the core of the standard Q-learning framework \citep{Watkins_Dayan_1992} for offline-RL is the $Q$-function $Q\!:\!\mathcal{S}\!\times\!\mathcal{A}_g\!\to\!\mathbb{R}$. Intuitively, $Q$-learning seeks to produce a policy $\pi^*(\cdot|s)$ that maximizes the expected infinite horizon discounted reward.  For any policy $\pi$, $Q^\pi(s,a)\!=\!\mathbb{E}^\pi[\sum_{t=0}^\infty\!\gamma^t r(s(t),a(t))|s(0)\!=\!s,a(0)\!=\!a]$.
One approach to learn the optimal policy $\pi^*(\cdot|s)$ is dynamic programming, where the $Q$-function is iteratively updated using value-iteration: $Q^0(s,a) = 0$, for all $(s,a)\in\mathcal{S}\times\mathcal{A}_g$. Then, for all $t\in[T]$, $Q^{t+1}(s,a)=\mathcal{T}Q^t(s,a)$, where $\mathcal{T}$ is the Bellman operator defined as $
\mathcal{T}Q^t(s,a)=r(s,a)+\gamma\E_{\substack{s_g'\sim P_g(\cdot|s_g,a), s_i'\sim P_l(\cdot|s_i,s_g),  \forall i\in[n]}} \max_{a'\in\mathcal{A}_g}Q^t(s',a')$. The Bellman operator
$\mathcal{T}$ satisfies a $\gamma$-contractive property, ensuring the existence of a unique fixed-point $Q^*$ such that $\mathcal{T}Q^* = Q^*$, by the Banach-Caccioppoli fixed-point theorem \citep{Banach1922}. Here, the optimal policy is the deterministic greedy policy $\pi^*\!\!:\!\!\mathcal{S}_g\!\times\mathcal{S}_l^n\to\mathcal{A}_g$, where $\pi^*(s)\!=\!\arg\max_{a\in\mathcal{A}_g} Q^*(s, a)$. However, in this solution, the complexity of a single update to the $Q$-function is $O(|\mathcal{S}_g||\mathcal{S}_l|^n|\mathcal{A}_g|)$, which grows exponentially with $n$. For practical purposes, even for small $n$, this complexity renders $Q$-learning impractical (see Example \ref{example: queueing}). 

\textbf{Mean-field Transformation. }To address this, \cite{pmlr-v80-yang18d,doi:10.1137/20M1360700} developed a mean-field approach which, under assumptions of homogeneity in the agents, considers the distribution function $F_{[n]}\!:\!\mathcal{S}_l\!\to\!\R$ given by $F_{[n]}(x)\!=\!\frac{\sum_{i=1}^n \1\{s_i=x\}}{n}$, for $x\!\in\! \mathcal{S}_l$. Define $\Theta_n\!=\!\{b/n:\!b\!\in\!\overline{[n]}\}$. With abuse of notation, let $F_{[n]}\!\in\!\Theta^{|\mathcal{S}_l|}$ be a vector storing the proportion of agents in each state. As the local agents are homogeneous, the $Q$-function is permutation-invariant in the local agents as permuting the labels of local agents with the same state will not change the global agent's decision. Hence, the $Q$-function only depends on $s_{[n]}$ through $F_{[n]}$: $Q(s_g,s_{[n]},a_g)\!=\!\hat{Q}(s_g,F_{[n]},a_g)$. Here, $\hat{Q}\!:\!\mathcal{S}_g\!\times\! \Theta^{|\mathcal{S}_l|}\!\times\!\mathcal{A}_g\!\to\!\R$ is a reparameterized $Q$-function learned by mean-field value iteration, where $\hat{Q}^0(s_g,F_{[n]},a_g)\!\!=\!\!0,\forall (s,a_g)\!\in\!\mathcal{S}\!\times\!\mathcal{A}_g$, and for all $t\in [T]$, $\hat{Q}^{t+1}(s,F_{[n]},a)\!=\! \hat{\mathcal{T}}\hat{Q}_k(s_g,F_{[n]},a)$. Here, $\hat{\mathcal{T}}$ is the Bellman operator in distribution space, which is given by \cref{eqn:bellman_op_dist_space}:
\begin{equation}\label{eqn:bellman_op_dist_space}\hat{\mathcal{T}}\hat{Q}^t(s_g,F_{[n]},a_g)\!=\!r(s,a)+\gamma\E_{s_g'\sim P_g(\cdot|s_g,a_g),s_i'\sim P_l(\cdot|s_i,s_g),\forall i\in[n]}\max_{a_g'\in\mathcal{A}_g}\hat{Q}^t(s',F_{[n]}',a_g').\end{equation} 
Then, since $\mathcal{T}$ has a $\gamma$-contractive property, so does $\hat{\mathcal{T}}$; hence $\hat{T}$ has a unique fixed-point $\hat{Q}^*$ such that $\hat{Q}^*(s_g,F_{{[n]}},a_g)=Q^*(s_g,s_{[n]},a_g)$. Finally, the optimal policy is the deterministic greedy policy $\hat{\pi}^*(s_g,F_{[n]})=\arg\max_{a_g\in\mathcal{A}_g}\hat{Q}^*(s_g,F_{[n]},a_g)$. Here, the complexity of a single update to the $\hat{Q}$-function is $O(|\mathcal{S}_g||\mathcal{A}_g|n^{|\mathcal{S}_l|})$, which scales polynomially in $n$.

However, for practical purposes, for larger values of $n$, the update complexity of mean-field value iteration can still be computationally intensive, and a subpolynomial-time policy learning algorithm would be desirable. Hence, we introduce the \texttt{SUBSAMPLE-Q} algorithm in \cref{section: main_results} to attain this.

\section{Method and Theoretical Results}
\label{section: main_results}
 \subsection{Proposed Method: \texttt{SUBSAMPLE-Q}}
In this work, we propose algorithm \texttt{SUBSAMPLE-Q} to overcome the $\mathrm{poly}(n)$ update time of mean-field $Q$-learning. In our algorithm, the global agent randomly samples a subset of local agents $\Delta\in\mathcal{U} \binom{[n]}{k}$ for $k\in[n]$. It ignores all other agents $[n]\setminus\Delta$ and uses an empirical mean-field value iteration to learn the $Q$-function $\hat{Q}^*_{k}$ and policy $\hat{\pi}_{k,m}^\est$ for this surrogate system of $k$ local agents. The surrogate reward gained by the system at each time step is $r_\Delta:\mathcal{S}\times\mathcal{A}_g\to\mathbb{R}$, given by \cref{equation: surrogate_rewards}:
\begin{equation}\label{equation: surrogate_rewards}r_\Delta(s, a_g)=r_g(s_g, a_g)+\frac{1}{|\Delta|}\sum_{i\in\Delta}r_l(s_g,s_i).\end{equation}
We then derive a randomized policy ${\pi}^\est_{k,m}$ which samples $\Delta\in\mathcal{U}\binom{[n]}{k}$ at each time-step to derive action $a_g\gets \hat{\pi}^\est_{k,m}(s_g,s_\Delta)$. We show that the policy ${\pi}_{k,m}^\est$ converges to the optimal policy $\pi^*$ as $k\to n$ and $m\to\infty$ in \cref{theorem: performance_difference_lemma_applied}. More formally, we present \cref{algorithm: approx-dense-tolerable-Q-learning} (\texttt{SUBSAMPLE-Q}: Learning) and \cref{algorithm: approx-dense-tolerable-Q*} (\texttt{SUBSAMPLE-Q}: Execution), which we describe below. A characterization that is crucial to our result is the notion of empirical distribution.

\begin{definition}[Empirical Distribution Function]\label{definition: empirical_distribution_function} \emph{For any population $(s_1, \dots, s_n) \in \mathcal{S}_l^n$, define the empirical distribution function $F_{s_\Delta}: \mathcal{S}_l \to \mathbb{R}$ for $\Delta\subseteq [n]$ by:}
\begin{equation}F_{s_\Delta}(x) := \frac{1}{|\Delta|}\sum_{i\in\Delta}\mathbbm{1}\{s_i = x\}.
\end{equation}
\end{definition}
Since the local agents in the system are homogeneous in their state spaces, transitions, and reward functions, the $Q$ function is permutation-invariant in the local agents as permuting the labels of local agents with the same state does not change the global agent's decision making process. Define $\Theta_k\!=\!\{b/k:\!b\!\in\!\overline{[k]}\}$. Then, $\hat{Q}_k$ depends on $s_\Delta$ through $F_{s_{\Delta}}\in \Theta_k^{|\mathcal{S}_l|}$. We denote this by \cref{equation:Q_n_emp}:
\begin{align}\label{equation:Q_n_emp}\hat{Q}_k(s_g, s_\Delta, a_g) = \hat{Q}_k(s_g, F_{s_\Delta}, a_g),\quad\quad Q(s_g, s_{[n]}, a_g) = \hat{Q}_n(s_g, F_{s_{[n]}}, a_g).\end{align}

\textbf{\cref{algorithm: approx-dense-tolerable-Q-learning}} (Offline learning). We empirically learn the optimal mean-field Q-function for a subsystem with $k$ local agents that we denote by $\hat{Q}^\est_{k,m}\!:\!\mathcal{S}_g\!\times\Theta_k^{|\mathcal{S}_l|}\!\times\!\mathcal{A}_g\!\to\!\mathbb{R}$, where $m$ is the sample size.
As in \cref{RL review}, we set $\hat{Q}_{k,m}^0(s_g,F_{s_\Delta},a_g)\!=\!0$ for all $s_g\!\in\!\mathcal{S}_g,F_{s_\Delta}\!\in\!\Theta_k^{|\mathcal{S}_l|},a_g\!\in\!\mathcal{A}_g$. For $t\!\in\!\N$, we set $\hat{Q}_{k,m}^{t+1}(s_g, F_{s_\Delta},a_g)\!=\!\hat{\mathcal{T}}_{k,m} \hat{Q}_{k,m}^t(s_g,F_{s_\Delta},a_g)$ where $\hat{\mathcal{T}}_{k,m}$ is the \emph{empirically adapted} Bellman operator defined for $k\!\leq\!n$ and $m\!\in\!\N$ in \cref{equation: empirical_adapted_bellman}. $\hat{\mathcal{T}}_{k,m}$ draws $m$ random samples $s_g^j\!\sim\!P_g(\cdot|s_g,a_g)$ for $j\!\in\![m]$ and $s_i^j\!\sim\! P_l(\cdot|s_i,s_g)$  for $j\!\in\![m]$, $i\!\in\!\Delta$. Here, the operator  $\hat{\mathcal{T}}_{k,m}$ is:
\begin{equation}
\label{equation: empirical_adapted_bellman}
\begin{split}
&\hat{\mathcal{T}}_{k,m}\hat{Q}_{k,m}^t(s_g,F_{s_\Delta},a_g) =r_\Delta(s,a_g)+\frac{\gamma}{m}\sum_{j\in[m]}\max_{a_g'\in\mathcal{A}_g}\hat{Q}_{k,m}^t(s_g^j,F_{s_\Delta^j},a_g').
\end{split}
\end{equation}
$\hat{\mathcal{T}}_{k,m}$ satisfies a $\gamma$-contraction property (see Lemma \ref{lemma: gamma-contraction of empirical adapted Bellman operator}). So, \cref{algorithm: approx-dense-tolerable-Q-learning} (\texttt{SUBSAMPLE-Q:} Learning) performs mean-field value iteration where it repeatedly applies $\hat{\mathcal{T}}_{k,m}$ to the same $\Delta\!\subseteq\![n]$ until $\hat{Q}_{k,m}$ converges to its fixed point $\hat{Q}_{k,m}^\est$ satisfying $\hat{\mathcal{T}}_{k,m}\hat{Q}_{k,m}^\est\!=\!\hat{Q}_{k,m}^\est$. We then obtain a deterministic policy $\hat{\pi}_{k,m}^\est\!:\!\mathcal{S}_g\!\times\!\Theta_k^{|\mathcal{S}_l|}$ given by $\hat{\pi}^\est_{k,m}(s_g,F_{s_\Delta}) = \arg\max_{a_g\in\mathcal{A}_g}\hat{Q}^\est_{k,m}(s_g,F_{s_\Delta},a_g)$.\\

\textbf{\cref{algorithm: approx-dense-tolerable-Q*}} (Online implementation). Here, \cref{algorithm: approx-dense-tolerable-Q*} (\texttt{SUBSAMPLE-Q}: Execution) randomly samples $\Delta\!\sim\!\mathcal{U}\binom{[n]}{k}$ at each time step and uses action $a_g\!\sim\! \hat{\pi}_{k,m}^\est(s_g,F_{s_\Delta})$ to get reward $r(s,a_g)$. This procedure of first sampling $\Delta$ and then applying $\hat{\pi}_{k,m}$ is denoted by a stochastic policy ${\pi}_{k,m}^\est(a_g|s)$:\begin{equation}{\pi}_{k,m}^\est(a_g|s) = \frac{1}{\binom{n}{k}}\sum_{\Delta\in \binom{[n]}{k}} \mathbbm{1}(\hat{\pi}_{k,m}^\est(s_g,F_{s_\Delta})=a_g ).\end{equation}
Then, each agent transitions to their next state based on \cref{equation: global_transition}. 

\begin{algorithm}[ht]
\caption{\texttt{SUBSAMPLE-Q}: Learning }\label{algorithm: approx-dense-tolerable-Q-learning}
\begin{algorithmic}[1]
\REQUIRE A multi-agent system as described in \cref{section: preliminaries}.
Parameter $T$ for the number of iterations in the initial value iteration step. Sampling parameters $k \in [n]$ and $m\in \N$. Discount parameter $\gamma\in (0,1)$. Oracle $\mathcal{O}$ to sample $s_g'\sim {P}_g(\cdot|s_g,a_g)$ and $s_i\sim {P}_l(\cdot|s_i,s_g)$ for all $i\in[n]$.
\STATE Uniformly choose $\Delta\subseteq [n]$ such that $|\Delta|=k$.
\STATE Set $\hat{Q}^0_{k,m}(s_g, F_{s_\Delta}, a_g)=0$, for $s_g\in\mathcal{S}_g, F_{s_\Delta}\in\Theta_k^{|\mathcal{S}_l|}, a_g\in\mathcal{A}_g$, where $\Theta_k=\{b/k: b\in\overline{[k]}\}$.
\FOR{$t=1$ to $T$}
\STATE $\hat{Q}^{t+1}_{k,m}(s_g, F_{s_\Delta}, a_g) = \hat{\mathcal{T}}_{k,m} \hat{Q}^t_{k,m}(s_g, F_{s_\Delta}, a_g)$, for all  $s_g\in \mathcal{S}_g, F_{s_\Delta}\in\Theta_k^{|\mathcal{S}_l|}, a_g \in \mathcal{A}_g$
\ENDFOR
\STATE For all $(s_g, F_{s_\Delta}) \in \mathcal{S}_g\times \Theta_k^{|\mathcal{S}_l|}$, let $\hat{\pi}_{k,m}^\est(s_g, F_{s_\Delta}) = \arg\max_{a_g\in\mathcal{A}_g}\hat{Q}_{k,m}^T(s_g, F_{s_\Delta}, a_g)$.
\end{algorithmic}
\end{algorithm}

\begin{algorithm}[ht]
\caption{\texttt{SUBSAMPLE-Q}: Execution} \label{algorithm: approx-dense-tolerable-Q*}
\begin{algorithmic}[1]
\REQUIRE 
A multi-agent system as described in \cref{section: preliminaries}. Parameter $T'$ for the number of rounds in the game. Hyperparameter $k\!\in\![n]$. Discount parameter $\gamma$. Policy $\hat{\pi}^\est_{k,m}(s_g, F_{s_\Delta})$.
\STATE Initialize $(s_g(0), s_{[n]}(0)) \sim s_0$, where $s_0$ is a distribution on the initial global state $(s_g,s_{[n]})$,
\STATE Initialize the total reward: $R_0\gets 0$.
\STATE \textbf{Policy} ${\pi}_{k,m}^\est(s)$:
\FOR{$t=0$ to $T'$}
\STATE Sample $\Delta$ uniformly at random from from $\binom{[n]}{k}$.
\STATE Let $a_g(t) \sim \hat{\pi}_{k,m}^\est(s_g(t), s_{\Delta}(t))$.
\STATE Let $s_g(t+1) \sim P_g(\cdot|s_g(t), a_g(t))$ and $s_i(t+1) \sim P_l(\cdot|s_i(t), s_g(t))$, for all $i\in [n]$.
\STATE $R_{t+1} = R_t + \gamma^t \cdot r(s, a_g)$
\ENDFOR
\end{algorithmic}
\end{algorithm}     

\begin{remark}
    \emph{\cref{algorithm: approx-dense-tolerable-Q-learning} assumes the existence of a generative model $\mathcal{O}$ \citep{Kearns_Singh_1998} to sample $s_g'\sim P_g(\cdot|s_g,a_g)$ and $s_i\sim P_l(\cdot|s_i,s_g)$. This is generalizable to the online reinforcement learning setting using techniques from \citep{NEURIPS2018_d3b1fb02}, and we leave this for future investigations. }
\end{remark}

\subsection{Theoretical Guarantee}
This subsection shows that the value of the expected discounted cumulative reward produced by ${\pi}^\est_{k,m}$ is approximately optimal, where the optimality gap decays as $k\!\to\!n$ and $m\!\to\!\infty$. 

\textbf{Bellman noise.} We first introduce the notion of Bellman noise, which is used in the main theorem. Firstly, clearly $\hat{\mathcal{T}}_{k,m}$ is an unbiased estimator of the generalized adapted Bellman operator $\hat{\mathcal{T}}_k$,
\begin{equation}
    \label{eqn:adapted bellman}
\begin{split}
&\hat{\mathcal{T}}_k\hat{Q}_k(s_g,F_{s_\Delta},a_g) \!=\!r_\Delta(s,a_g)\!+\!\gamma\E_{\substack{s_g'\sim P_g(\cdot|s_g,a_g), s_i'\sim P_l(\cdot|s_i,s_g), \forall i\in\Delta}}\max_{a_g'\in\mathcal{A}_g}\hat{Q}_k(s_g',F_{s_\Delta'},a_g').
\end{split}
\end{equation}
For all $s_g\!\in\!\mathcal{S}_g, F_{s_\Delta}\!\in\!\Theta_k^{|\mathcal{S}_l|}, a_g\!\in\!\mathcal{A}_g$,  $\hat{Q}_k(s_g,F_{s_\Delta},a_g)=0$. For $t\!\in\!\N$, let $\hat{Q}_k^{t+1}\!=\!\hat{\mathcal{T}}_k\hat{Q}_k^t$, where $\hat{\mathcal{T}}_k$ is defined for $k\leq n$ in \cref{eqn:adapted bellman}. Similarly to $\hat{\mathcal{T}}_{k,m}$, $\hat{\mathcal{T}}_k$ satisfies a $\gamma$-contraction property (\cref{lemma: gamma-contraction of adapted Bellman operator}) with fixed-point $\hat{Q}_k^*$. By the law of large numbers,  $\lim_{m\to\infty}\hat{\mathcal{T}}_{k,m}\!=\!\hat{\mathcal{T}}_k$. Hence, the gap $\|\hat{Q}_{k,m}^\est - \hat{Q}_k^*\|_\infty$ converges to $0$ as $m\rightarrow \infty$. For finite $m$, $\|\hat{Q}_{k,m}^\est - \hat{Q}_k^*\|_\infty=:\epsilon_{k,m}$ is called the Bellman noise. Bounding $\epsilon_{k,m}$ has been well studied in the literature. One such bound is:
\begin{lemma}[Theorem 1 of \cite{9570295}] \label{assumption:qest_qhat_error}\emph{For all $k\in [n]$ and $m\in\N$, where $m$ is the number of samples in \cref{equation: empirical_adapted_bellman}, there exists a Bellman noise $\epsilon_{k,m}$ such that $\|\hat{\mathcal{T}}_{k,m}\hat{Q}_{k,m}^{\est} - \hat{\mathcal{T}}_k\hat{Q}_k^*\|_\infty = \|\hat{Q}_{k,m}^\est - \hat{Q}_k^*\|_\infty \leq \epsilon_{k,m}\leq O(1/\sqrt{m}).$}
\end{lemma}

With the above preparations, we are now primed to present our main result: a bound on the optimality gap for our learned policy ${\pi}_{k,m}^\est$ that decays with $k$. \cref{section: proof_outline} outlines the proof of \cref{theorem: performance_difference_lemma_applied}.
\begin{theorem} \emph{For any state $s\in\mathcal{S}_g\times\mathcal{S}_l^n$},\label{theorem: performance_difference_lemma_applied}
\begin{align*}V^{\pi^*}(s) - V^{{{\pi}}^\est_{k,m}}(s) &\leq \frac{2\tilde{r}}{(1-\gamma)^2}\left(\sqrt{\frac{n-k+1}{2nk} \ln(2|\mathcal{S}_l||\mathcal{A}_g|\sqrt{k})}+\frac{1}{\sqrt{k}}\right)+\frac{2\epsilon_{k,m}}{1-\gamma}.\end{align*}
\end{theorem}

\begin{corollary}\label{remark: pdl applied}
\emph{\cref{theorem: performance_difference_lemma_applied} implies an asymptotically decaying optimality gap for our learned policy $\tilde{\pi}_{k,m}^\est$. Further, from Lemma \ref{assumption:qest_qhat_error}, $\epsilon_{k,m} \leq O(1/\sqrt{m})$. Hence, }
\begin{equation}V^{\pi^*}(s) - V^{{{\pi}}^\est_{k,m}}(s) \leq \tilde{O}\left(1/\sqrt{k} + 1/\sqrt{m}\right).
\end{equation}
\end{corollary}

\begin{discussion}\label{disussion: complexity requirement}
\emph{The size of $\hat{Q}_{k,m}(s_g,F_{s_\Delta},a_g)$ is $O(|\mathcal{S}_g| |\mathcal{A}_g| k^{|\mathcal{S}_l|})$. From \cref{theorem: performance_difference_lemma_applied}, as $k\to n$, the optimality gap decays, revealing a trade-off in the choice of $k$, between the size of the $Q$-function and the optimality of the policy ${\pi}_{k,m}^\est$. We demonstrate this trade-off further in our experiments. For $k=O(\log n)$ and $m\to\infty$, we get an exponential speedup  on the complexity from mean-field value iteration (from $\mathrm{poly}(n)$ to $\mathrm{poly}(\log n)$), and a super-exponential speedup  from traditional value-iteration (from $\mathrm{exp}(n)$ to $\mathrm{poly}(\log n)$, with a decaying $O(1/\sqrt{\log n})$ optimality gap. This gives a competitive policy-learning algorithm with \emph{polylogarithmic} run-time.}
\end{discussion}

\begin{discussion}\emph{
    One could replace the $Q$-learning algorithm with an arbitrary value-based RL method that learns $\hat{Q}_k$ with function approximation \citep{NIPS1999_464d828b} such as deep $Q$-networks  \citep{Silver_Huang_Maddison_Guez_Sifre_van_den_Driessche_Schrittwieser_Antonoglou_Panneershelvam_Lanctot_et_al._2016}. Doing so introduces a further error that factors into the bound in \cref{remark: pdl applied}.}
\end{discussion}

\section{Proof Outline}
\label{section: proof_outline}
This section details an outline for the proof of \cref{theorem: performance_difference_lemma_applied}, as well as some key ideas. At a high level, our \texttt{SUBSAMPLE-Q} framework in \cref{algorithm: approx-dense-tolerable-Q-learning,algorithm: approx-dense-tolerable-Q*} recovers exact mean-field $Q$ learning (and therefore, traditional value iteration) when $k\!\!=\!\!n$ and as $m\!\!\to\!\!\infty$. Further, as $k\!\!\to\!\!n$, $\hat{Q}_k^*$ should intuitively get closer to $Q^*$ from which the optimal policy is derived. Thus, the proof is divided into three steps. We first prove a Lipschitz continuity bound between $\hat{Q}_k^*$ and $\hat{Q}_n^*$ in terms of the total variation (TV) distance between $F_{s_\Delta}$ and $F_{s_{[n]}}$. Secondly, we bound the TV distance between $F_{s_\Delta}$ and $F_{s_{[n]}}$. Finally, we bound the value differences between $\tilde{\pi}_{k,m}^\est$ and $\pi^*$ by bounding $Q^*(s,\pi^*)\!-\!Q^*(s,\hat{\pi}_{k,m}^\est)$ and then using the performance difference lemma \citep{Kakade+Langford:2002}.

\textbf{Step 1: Lipschitz Continuity Bound.}
To compare $\hat{Q}_k^*(s_g,F_{s_\Delta},a_g)$ with $Q^*(s,a_g)$, we prove a Lipschitz continuity bound between $\hat{Q}^*_k(s_g, F_{s_\Delta},a_g)$ and $\hat{Q}^*_{k'}(s_g, F_{s_{\Delta'}},a_g)$ with respect to the TV distance measure between $s_\Delta\in \binom{s_{[n]}}{k}$ and $s_{\Delta'}\in \binom{s_{[n]}}{k'}$. Specifically, we show:

\begin{theorem}[Lipschitz continuity in $\hat{Q}_k^*$]\label{thm:lip}
 For all $(s, a) \in\mathcal{S}\times\mathcal{A}_g$, $\Delta\in\binom{[n]}{k}$ and $\Delta'\in\binom{[n]}{k'}$, \begin{align*}|\hat{Q}^*_k(s_g,F_{s_\Delta},a_g) &- \hat{Q}^*_{k'}(s_g, F_{s_{\Delta'}}, a_g)|  \leq {2(1-\gamma)^{-1}\|r_l(\cdot,\cdot)\|_\infty }\cdot \mathrm{TV}\left(F_{s_\Delta}, F_{s_{\Delta'}}\right)
\end{align*}
\end{theorem}
We defer the proof of \cref{thm:lip} to Appendix \ref{theorem: Q-lipschitz of Fsdelta and Fsn}. See \cref{figure: algorithm/analysis flow}
for a comparison between the $\hat{Q}_k^*$ learning and estimation process, and the exact ${Q}$-learning framework.

\paragraph{Step 2: Bounding Total Variation (TV) Distance.} We bound the TV distance between $F_{s_\Delta}$ and $F_{s_{[n]}}$, where $\Delta\!\in\!\mathcal{U}\binom{[n]}{k}$. Bounding this TV distance is equivalent to bounding the discrepancy between the empirical distribution and the distribution of the underlying finite population. Since each $i\!\in\!\Delta$ is chosen uniformly at random and \emph{without} replacement, standard concentration inequalities do not apply as they require the random variables to be i.i.d. Further, standard TV distance bounds that use the KL divergence produce a suboptimal decay as $|\Delta|\to n$ (Lemma \ref{lemma: tv_distance_bretagnolle_huber}). Therefore, we prove the following probabilistic result (which generalizes the Dvoretzky–Kiefer–Wolfowitz (DKW) concentration inequality \citep{10.1214/aoms/1177728174} to the regime of sampling \emph{without} replacement:

\begin{theorem}\label{thm:tvd}
Given a finite population $\mathcal{X}=(x_1,\dots,x_n)$ for $\mathcal{X}\in\mathcal{S}_l^n$, let $\Delta\subseteq[n]$ be a uniformly random sample from $\mathcal{X}$ of size $k$ chosen without replacement. Fix $\epsilon>0$. Then, for all $x\in\mathcal{S}_l$:
\begin{align*}\Pr\bigg[\sup_{x\in\mathcal{S}_l}\bigg|\frac{1}{|\Delta|}\sum_{i\in\Delta}\mathbbm{1}{\{x_i = x\}} &- \frac{1}{n}\sum_{i\in[n]}\mathbbm{1}{\{x_i = x\}}\bigg|\leq \epsilon\bigg] \geq 1 - 2|\mathcal{S}_l|e^{-\frac{2|\Delta|n\epsilon^2}{n-|\Delta|+1}}.\end{align*}
\end{theorem}
Then, by \cref{thm:tvd} and the definition of TV distance from \cref{section: preliminaries}, we have that for $\delta\in(0,1]$,
\begin{equation}\Pr\left(\mathrm{TV}(F_{s_\Delta}, F_{s_{[n]}})\!\leq\!\sqrt{\frac{n\!-\!|\Delta|\!+\!1}{8n|\Delta|}\ln \frac{ 2|\mathcal{S}_l|}{\delta}} \right)\geq 1-\delta.\end{equation}
We then apply this result to our global decision-making problem by studying the rate of decay of the objective function between our learned policy ${\pi}_{k,m}^\est$ and the optimal policy $\pi^*$ (\cref{theorem: performance_difference_lemma_applied}). 

\textbf{Step 3: Performance Difference Lemma to Complete the Proof.} As a consequence of the prior two steps and Lemma \ref{assumption:qest_qhat_error}, $Q^*(s,a'_g)$ and $\hat{Q}_{k,m}^\est(s_g,F_{s_\Delta},a_g')$ become similar as $k\to n$ (see \cref{theorem: Q-lipschitz of Fsdelta and Fsn}). We further prove that the value generated by their policies $\pi^*$ and ${\pi}_{k,m}^\est$ must also be very close (where the residue shrinks as $k \to n$). We then use the well-known performance difference lemma \citep{Kakade+Langford:2002} which we restate and explain in \ref{theorem: performance difference lemma} in the appendix. A crucial theorem needed to use the performance difference lemma is a bound on $Q^*(s',\pi^*(s')) - Q^*(s',\hat{\pi}_{k,m}^\est(s_g',s_\Delta'))$. Therefore, we formulate and prove \cref{thm:q_diff_actions} which yields a probabilistic bound on this difference, where the randomness is over the choice of $\Delta\in\binom{[n]}{k}$.

\begin{theorem}\label{thm:q_diff_actions}
For a fixed $s'\in\mathcal{S}:=\mathcal{S}_g\times\mathcal{S}_l^n$ and for $\delta\in (0,1]$, with probability atleast $1 - 2|\mathcal{A}_g|\delta$:
\begin{align*}Q^*(s',\pi^*(s')) - Q^*(s',\hat{\pi}_{k,m}^\est(s_g',F_{s_\Delta'}))\leq\frac{2\|r_l(\cdot,\cdot)\|_\infty}{1-\gamma}\sqrt{\frac{n-k+1}{2nk}\ln\left( \frac{2|\mathcal{S}_l|}{\delta}\right)}+ 2\epsilon_{k,m}.
\end{align*}
\end{theorem}
We defer the proof of \cref{thm:q_diff_actions}
and finding optimal value of the parameters $\delta_1,\delta_2$ to \ref{lemma: q_star_different_action_bounds} in the Appendix. Using \cref{thm:q_diff_actions} and the performance difference lemma leads to \cref{theorem: performance_difference_lemma_applied}.

\section{Experiments}
\label{section: experiments}
This section provides examples and numerical simulation results to validate our theoretical framework. All numerical experiments were run on a 3-core CPU server equipped with a 12GB RAM. We chose parameters with complexity sufficient to only validate the theory, such as the computational speedups, pseudo-heterogeneity of each local agent, and the decaying optimality gap.

\begin{example}[Demand-Response (DR)]
\label{example: demand-response}
\emph{DR is a pathway in the transformation towards a sustainable electricity grid where users (local agents) are compensated to lower their electricity consumption to a level set by a regulator (global agent). DR has applications ranging from pricing strategies for EV charging stations, regulating the supply of any product in a market with fluctuating demands, and maximizing the efficiency of allocating resources. We ran a small-scale simulation with $n=8$ local agents, and a large-scale simulation with $n=50$ local agents, where the goal was to learn an optimal policy for the global agent to moderate supply in the presence of fluctuating demand.}

\emph{Let each local agent $i\in[n]$ have a state $s_i(t)=(\psi_i,s_i^*(t),\bar{s}_i(t))\in\mathcal{S}_l:=\Psi\times\mathcal{D}_a\times\mathcal{D}_c\subseteq\mathbb{Z}^3$. Here, $\psi_i$ is the agent's type, $s_i^*(t)$ is agent $i$'s consumption, and $\bar{s}_i(t)$ is its desired consumption level. Let $s_g(t)\!\in\!\mathcal{S}_g,a_g(t)\!\in\!\mathcal{A}_g$ where $s_g(t)$ is the DR signal (target consumption set by the regulator). The global agent transition is given by $s_g(t\!+\!1)\!=\!\Pi^{\mathcal{S}_g}(s_g(t)\!+\!a_g(t))$,
i.e., $a_g(t)$ changes the DR signal. Then, $s_i(t+1)\!=\!(\psi_i, \bar{s}_i(t\!+\!1),s_i^*(t\!+\!1))$, where intuitively, $\bar{s}_i(t\!+\!1)$ fluctuates based on $\psi_i$ and $\bar{s}_i(t)$. If $\bar{s}_i(t)\!<\!s_g(t)$, then $s_i^*(t\!+\!1)\!=\!\bar{s}_i(t)$ (the local agent chases its desired consumption). If not, the local agent either follows $\bar{s}_i(t)$ or reduces its consumption to match $s_g(t)$. Formally, if $\psi_i=1$, then $\bar{s}_i(t+1)=\bar{s}_i(t)+\mathcal{U}\{0,1\}$. If $\psi_i=2,\bar{s}_i(t+1)=\mathcal{U}\{\mathcal{D}_c\}$. If $\bar{s}_i(t)\leq s_g(t)$, then $\bar{s}_i^*(t+1)=\bar{s}_i(t)$. If $\bar{s}_i(t)>s_g(t)$, then $\bar{s}_i^*(t+1)=\Pi^{\mathcal{D}_c}[\bar{s}_i(t) + (s_g(t) - s_i^*(t))\mathcal{U}\{0,1\}]$.
The reward of the system at each step is given by $r_g(s_g,a_g)\!=\!15/s_g-\mathbbm{1}\{a_g\!=\!-1\}$ and $r_l(s_i,s_g)=s_i^*-\frac{1}{2}\mathbbm{1}\{s_i^*>s_g\}$. We set $\mathcal{D}_a=\mathcal{D}_c=[5], \Psi=\{1,2\}, \gamma=0.9, m=50$, and the length of the decision game to be $T'=300$. }

\emph{We use $T=300$ empirical adapted Bellman iterations for the small-scale simulation, and $T=50$ iterations for the large scale simulation. For the small-scale simulation, \cref{demand-response figures}a illustrates the polynomial speedup of \cref{algorithm: approx-dense-tolerable-Q-learning} (note that $k=n$ exactly recovers mean-field value iteration \cite{pmlr-v80-yang18d}, which we treat as our baseline comparison). \cref{demand-response figures}b plots the reward-optimality gap for varying $k$, illustrating that the gap decreases monotonically as $k \to n$, as shown in \cref{theorem: performance_difference_lemma_applied}. \cref{demand-response figures}c plots the cumulative reward of the large-scale experiment. We observe that the rewards (on average) grow monotonically as they obey our worst-case guarantee in \cref{theorem: performance_difference_lemma_applied}.}
\end{example}

\begin{example}[Queueing]\emph{
    \label{example: queueing}  
    We model a system with $n$ queues, $s_i(t)\in\mathcal{S}_l:=\N$ at time $t$ denotes the number of jobs at time $t$ for queue $i\in[n]$. We model the job allocation mechanism as a global agent where $s_g(t) \in \mathcal{S}_g = \mathcal{A}_g = [n]$, where $s_g(t)$ denotes the queue to which the next job should be delivered. We choose the state transitions to capture the stochastic job arrival and departure: $s_g(t+1)=a_g(t)$, and
$s_i(t+1)=\min\{c,\max\{0,s_i(t)+\mathbbm{1}\{s_g(t)=i\}-\mathrm{Bern}(p)\}\}$. For the rewards, we set $r_g(s_g(t), a_g(t)) = 0, r_l(s_i(t),s_g(t))= -s_i(t) - 10\cdot \mathbbm{1}\{s_i(t)>c\}$, where $p=0.8$ is the probability of finishing a job, $c=30$ is the capacity of each queue, and $\gamma=0.9$.}

\emph{This simulation ran on a system of $n=50$ local agents. The goal was to learn an optimal policy for a dispatcher to send incoming jobs to. We ran \cref{algorithm: approx-dense-tolerable-Q-learning} for $T=300$ empirical adapted Bellman iterations with $m=30$, and ran \cref{algorithm: approx-dense-tolerable-Q*} for $T'=100$ iterations. \cref{figure: queueing_theory_simulation} illustrates the log-scale reward-optimality gap for varying $k$, showing  that the gap decreases monotonically as $k \to n$ with a decay rate that is consistent with the $O(1/\sqrt{k})$ upper bound in \cref{theorem: performance_difference_lemma_applied}.}
\end{example}

\label{experiment1}
\begin{figure*}
\begin{center}
\includegraphics[width=0.34\linewidth]{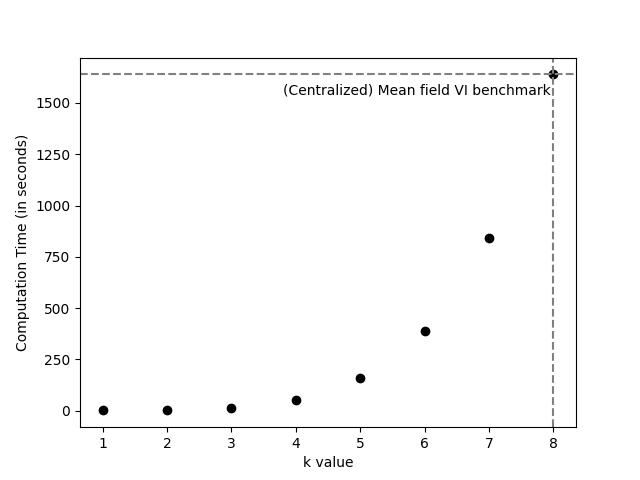}
\includegraphics[width=0.32\linewidth]{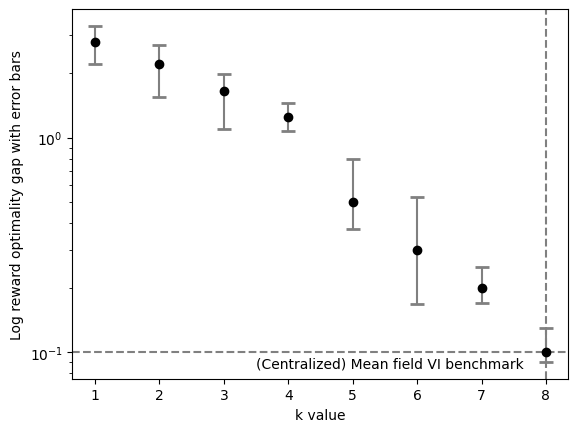}
\includegraphics[width=0.32\linewidth]{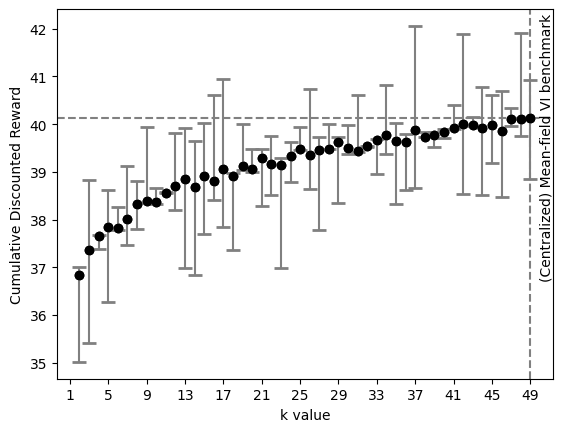}
\end{center}
\caption{Demand-Response simulation.
a) Computation time to learn $\hat{\pi}_{k,m}^\est$ for $k\!\leq\!n\!=\!8$. b) Reward optimality gap (log scale) with ${\pi}_{k,m}^\est$ running $300$ iterations for $k\leq n=8$, c) Discounted cumulative rewards for $k\!\leq\!n\!=\!50$. We note that $k\!=\!n$ recovers the mean-field RL iteration solution.}
\label{demand-response figures}
\end{figure*}
\label{experiment2}
\begin{figure}[h]
\begin{center}
\centerline{\includegraphics[scale=0.34]{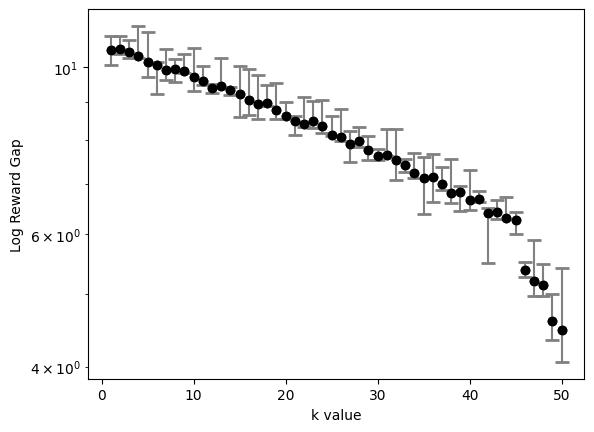}}
\caption{Reward optimality gap (log scale) with ${\pi}_{k,m}^\est$ running $300$ iterations.}
\label{figure: queueing_theory_simulation}
\end{center}
\end{figure}

\section{Conclusion, Limitations, and Future Work}
\textbf{Conclusion.} This work considers a global decision-making agent in the presence of $n$ local homogeneous agents. We propose \texttt{SUBSAMPLE-Q} which derives a policy ${\pi}_{k,m}^\est$ where $k\leq n$ and $m\in\N$ are tunable parameters, and show that ${\pi}_{k,m}^\est$ converges to the optimal policy $\pi^*$ with a decay rate of $O(1/\sqrt{k} + \epsilon_{k,m})$, where $\epsilon_{k,m}$ is the Bellman noise. To establish the result, we develop an adapted Bellman operator $\hat{\mathcal{T}}_k$ and  show a Lipschitz-continuity result for $\hat{Q}_k^*$ and generalize the DKW inequality. Finally, we validate our theoretical result through numerical experiments.

\textbf{Limitations and Future Work.} We recognize several future directions. This model studies a ‘star-graph’ setting to model a single source of density. It would be fascinating to extend to general graphs. We believe expander-graph decomposition methods \citep{anand2023pseudorandomness} are amenable for this. Another direction is to find connections between our sub-sampling method to algorithms in federated learning, where the rewards can be stochastic and to incorporate learning rates \cite{DBLP:conf/nips/LinQHW21} to attain numerical stability. Another limitation of this work is that we have only partially resolved the problem for truly heterogeneous local agents by adding a `type' property to each local agent to model some pseudoheterogeneity in the state space of each agent. Additionally, it would be interesting to extend this work to the online setting without a generative oracle simulator. Finally, our model assumes finite state/action spaces as in the fundamental tabular MDP setting. However, to increase the applicability of the model, it would be interesting to replace the $Q$-learning algorithm with a deep-$Q$ learning or a value-based RL method where the state/action spaces can be continuous. 

\section{Acknowledgements}
This work is supported by a research assistantship at Carnegie Mellon University and a fellowship from the Caltech Associates. We thank ComputeX for allowing usage of their server to run numerical experiments and gratefully acknowledge insightful conversations with Yiheng Lin, Ishani Karmarkar, Elia Gorokhovsky, David Hou, Sai Maddipatla, Alexis Wang, and Chris Zhou.

\newpage

\bibliography{iclr2025_conference}
\bibliographystyle{iclr2025_conference}
\newpage
\appendix

\textbf{Outline of the Appendices}.
\begin{itemize}
    \item Appendix A presents additional definitions and remarks that support the main body.
    \item Appendix B-C contains a detailed proof of the Lipschitz continuity bound in \cref{thm:lip} and total variation distance bound in \cref{thm:tvd}.
    \item Appendix D contains a detailed proof of the main result in \cref{theorem: performance_difference_lemma_applied}.
\end{itemize}

\section{Mathematical Background and Additional Remarks}
\label{appendix: math_background}
\begin{definition}[Lipschitz continuity] \emph{Given two metric spaces $(\mathcal{X}, d_\mathcal{X})$ and $(\mathcal{Y}, d_\mathcal{Y})$ and a constant $L\in\mathbb{R}_+$, a mapping $f:\mathcal{X}\to \mathcal{Y}$ is $L$-Lipschitz continuous if for all $x,y\in \mathcal{X}$, $
    d_\mathcal{Y}(f(x), f(y)) \leq L \cdot  d_\mathcal{X}(x,y)$.}
\end{definition} 

\begin{theorem}[Banach-Caccioppoli fixed point theorem \cite{Banach1922}]
\emph{Consider the metric space $(\mathcal{X}, d_\mathcal{X})$, and $T: \mathcal{X}\to \mathcal{X}$ such that $T$ is a $\gamma$-Lipschitz continuous mapping for $\gamma \in (0,1)$. Then, by the Banach-Cacciopoli fixed-point theorem, there exists a unique fixed point $x^* \in \mathcal{X}$ for which $T(x^*) = x^*$. Additionally, $x^* = \lim_{s\to\infty} T^s( x_0)$ for any $x_0 \in \mathcal{X}$.}
\end{theorem}


For convenience, we restate below the various Bellman operators under consideration.

\begin{definition}[Bellman Operator $\mathcal{T}$]\label{defn:bellman}
\begin{equation}
\mathcal{T}Q^t(s,a_g) := r_{[n]}(s,a_g) + \gamma\E_{\substack{s_g'\sim P_g(\cdot|s_g,a_g),\\ s_i'\sim P_l(\cdot|s_i,s_g),\forall i\in[n]}} \max_{a_g'\in\mathcal{A}_g} Q^t(s',a_g')
\end{equation}
\end{definition}
\begin{definition}[Adapted Bellman Operator $\hat{\mathcal{T}}_k$]\label{defn:adapted bellman} \emph{The adapted Bellman operator updates a smaller $Q$ function (which we denote by $\hat{Q}_k$), for a surrogate system with the global agent and $k\in[n]$ local agents, using mean-field value iteration:}
\begin{equation}
\hat{\mathcal{T}}_k\hat{Q}_k^t(s_g,F_{s_\Delta},a_g):= r_\Delta(s,a_g) + \gamma \E_{\substack{s_g'\sim P_g(\cdot|s_g,a_g), \\ s_i'\sim P_l(\cdot|s_i,s_g),\forall i\in\Delta}} \max_{a_g'\in\mathcal{A}_g} \hat{Q}_k^t(s_g',F_{s_\Delta'},a_g')
\end{equation}
\end{definition}
\begin{definition}[Empirical Adapted Bellman Operator $\hat{\mathcal{T}}_{k,m}$] \label{defn:empirical adapted bellman}\emph{The empirical adapted Bellman operator $\hat{\mathcal{T}}_{k,m}$ empirically estimates the adapted Bellman operator update using mean-field value iteration by drawing $m$ random samples of $s_g\sim P_g(\cdot|s_g,a_g)$ and $s_i\sim P_l(\cdot|s_i,s_g)$ for $i\in\Delta$, where for $j\in[m]$, the $j$'th random sample is given by $s_g^j$ and $s_\Delta^j$: }
\begin{equation}
\hat{\mathcal{T}}_{k,m}\hat{Q}_{k,m}^t(s_g,F_{s_\Delta},a_g):= r_\Delta(s,a_g) + \frac{\gamma}{m} \sum_{j \in [m]}\max_{a_g'\in\mathcal{A}_g} \hat{Q}_{k,m}^t(s_g^j,F_{s_\Delta^j},a_g')
\end{equation}
\end{definition}
\begin{remark}
\emph{\label{remark: comparing bellman variants} We remark on the following relationships between the variants of the Bellman operators from \cref{defn:bellman,defn:adapted bellman,defn:empirical adapted bellman}. First, by the law of large numbers, we have $\lim_{m\to\infty}\hat{\mathcal{T}}_{k,m} = \hat{\mathcal{T}}_k$, where the error decays in $O(1/\sqrt{m})$ by the Chernoff bound. 
 Secondly, by comparing \cref{defn:adapted bellman} and \cref{defn:bellman}, we have $\mathcal{T}_n = \mathcal{T}$.}\\
\end{remark}

\begin{lemma}
\label{lemma: Q-bound}
    For any $\Delta\subseteq[n]$ such that $|\Delta|=k$, suppose $0\leq r_\Delta(s,a_g)\leq \tilde{r}$. Then, $\hat{Q}_k^t \leq \frac{\tilde{r}}{1-\gamma}$.
\end{lemma}
\begin{proof}
    We prove this by induction on $t\in\mathbb{N}$. The base case is satisfied as $\hat{Q}_k^0 = 0$. Assume that $\|\hat{Q}_k^{t-1}\|_\infty \leq \frac{\tilde{r}}{1-\gamma}$. We bound $\hat{Q}_k^{t+1}$ from the Bellman update at each time step as follows, for all $s_g\in\mathcal{S}_g,F_{s_\Delta}\in \Theta_k^{|\mathcal{S}_l|}, a_g\in\mathcal{A}_g$:
    \begin{align*}
        \hat{Q}_k^{t+1}(s_g,F_{s_\Delta},a_g) &= r_\Delta(s,a_g) + \gamma\mathbb{E}_{\substack{s_g'\sim P_g(\cdot|s_g, a_g), \\ s_i' \sim P_l(\cdot|s_i,s_g), \forall i\in\Delta}}\max_{a_g'\in\mathcal{A}_g}\hat{Q}_k^t(s_g',F_{s_\Delta'},a_g') \\
        &\leq \tilde{r} + \gamma \max_{\substack{a_g'\in\mathcal{A}_g, s_g' \in \mathcal{S}_g, F_{s_\Delta'} \in \Theta_k^{|\mathcal{S}_l|}}}\hat{Q}_k^t(s_g',F_{s_\Delta'},a_g') \leq \frac{\tilde{r}}{1-\gamma}
    \end{align*}
Here, the first inequality follows by noting that the maximum value of a random variable is at least as large as its expectation. The second inequality follows from the inductive hypothesis.\qedhere
\end{proof}

\begin{remark}\cref{lemma: Q-bound} is independent of the choice of $k$. Therefore, for $k=n$, this implies an identical bound on $Q^t$. A similar argument as \cref{lemma: Q-bound} implies an identical bound on $\hat{Q}_{k,m}^t$.
\end{remark}

Recall that the original Bellman operator $\mathcal{T}$ satisfies a $\gamma$-contractive property under the infinity norm. We similarly show that $\hat{\mathcal{T}}_k$ and $\hat{\mathcal{T}}_{k,m}$ satisfy a $\gamma$-contractive property under infinity norm in \cref{lemma: gamma-contraction of adapted Bellman operator} and \cref{lemma: gamma-contraction of empirical adapted Bellman operator}.

\begin{lemma}\label{lemma: gamma-contraction of adapted Bellman operator}
    $\hat{\mathcal{T}}_k$ satisfies the $\gamma$-contractive property under infinity norm:
    \[\|\hat{\mathcal{T}}_k\hat{Q}_k' - \hat{\mathcal{T}}_k\hat{Q}_k\|_\infty \leq \gamma \|\hat{Q}_k' - \hat{Q}_k\|_\infty\]
\end{lemma}
\begin{proof} Suppose we apply $\hat{\mathcal{T}}_k$ to $\hat{Q}_k(s_g,F_{s_\Delta}, a_g)$ and $\hat{Q}'_k(s_g, F_{s_\Delta}, a_g)$ for $|\Delta|=k$. Then:
\begin{align*}
&\|\hat{\mathcal{T}}_k\hat{Q}_k'- \hat{\mathcal{T}}_k\hat{Q}_k\|_\infty \\
&= \gamma \max_{\substack{s_g\in \mathcal{S}_g,\\ a_g\in\mathcal{A}_g,\\ F_{s_\Delta} \in \Theta_k^{|\mathcal{S}_l|}}}\!\left| \mathbb{E}_{\substack{s_g'\sim P_g(\cdot|s_g,a_g),\\ s_i'\sim P_l(\cdot|s_i, s_g),\\ \forall s_i'\in s_\Delta',\\ }}\max_{a_g'\in\mathcal{A}_g}\hat{Q}_k'(s_g', F_{s_\Delta'}, a_g')- \mathbb{E}_{\substack{s_g'\sim P_g(\cdot|s_g,a_g),\\ s_i'\sim P_l(\cdot|s_i, s_g),\\\forall s_i'\in s_\Delta'
}}\max_{a_g'\in\mathcal{A}_g}\hat{Q}_k(s_g', F_{s_\Delta'}, a_g')\right|\\
    &\leq \gamma  \max_{\substack{s_g' \in \mathcal{S}_g, F_{s_\Delta'} \in \Theta_k^{|\mathcal{S}_l|}, a_g'\in\mathcal{A}_g
        }}\left| \hat{Q}_k'(s_g', F_{s_\Delta'}, a_g') -  \hat{Q}_k(s_g', F_{s_\Delta'}, a_g')\right| \\
        &= \gamma \|\hat{Q}_k' - \hat{Q}_k\|_\infty
    \end{align*}
The equality implicitly cancels the common $r_\Delta(s, a_g)$ terms from each application of the adapted-Bellman operator. The inequality follows from Jensen's inequality, maximizing over the actions, and bounding the expected value with the maximizers of the random variables. The last line recovers the definition of infinity norm. \qedhere
\end{proof}

\begin{lemma}
$\hat{\mathcal{T}}_{k,m}$ satisfies the $\gamma$-contractive property under infinity norm.\label{lemma: gamma-contraction of empirical adapted Bellman operator}
\end{lemma}
\begin{proof}
Similarly to \cref{lemma: gamma-contraction of adapted Bellman operator}, suppose we apply $\hat{\mathcal{T}}_{k,m}$ to $\hat{Q}_{k,m}(s_g,F_{s_\Delta},a_g)$ and $\hat{Q}_{k,m}'(s_g,F_{s_\Delta},a_g)$. Then:
\begin{align*}    \|\hat{\mathcal{T}}_{k,m}\hat{Q}_k - \hat{\mathcal{T}}_{k,m}\hat{Q}'_k\|_\infty &= \frac{\gamma}{m}\left\|\sum_{j\in [m]} (\max_{a_g'\in\mathcal{A}_g} \hat{Q}_k(s_g^j,F_{s_\Delta^j},a_g') -  \max_{a_g'\in\mathcal{A}_g} \hat{Q}'_k(s_g^j,F_{s_\Delta^j},a_g'))\right\|_\infty \\
    &\leq \gamma \max_{\substack{a_g'\in\mathcal{A}_g, s_g' \in\mathcal{S}_g, s_\Delta\in\mathcal{S}_l^k}}|\hat{Q}_k(s_g', F_{s_\Delta'}, a_g') - \hat{Q}'_k(s_g', F_{s_\Delta'}, a_g')| \\
    &\leq \gamma \|\hat{Q}_k - \hat{Q}'_k\|_\infty
\end{align*} 
 The first inequality uses the triangle inequality and the general property $|\max_{a\in A}f(a) - \max_{b\in A}f(b)| \leq \max_{c\in A}|f(a) - f(b)|$. In the last line, we recover the definition of infinity norm.\qedhere
\end{proof}

\begin{remark}\label{remark: gamma-contractive of ad_T}\emph{Intuitively, the $\gamma$-contractive property of $\hat{\mathcal{T}}_k$ and $\hat{\mathcal{T}}_{k,m}$ causes the trajectory of two $\hat{Q}_k$ and $\hat{Q}_{k,m}$ functions on the same state-action tuple to decay by $\gamma$ at each time step such that repeated applications of their corresponding Bellman operators produce a unique fixed-point from the Banach-Cacciopoli fixed-point theorem which we introduce in \cref{defn:qkstar,defn:qkmest}.} \end{remark}

\begin{definition}[$\hat{Q}_k^*$]\label{defn:qkstar}Suppose $\hat{Q}_k^0:=0$ and let $\hat{Q}_k^{t+1}(s_g,F_{s_\Delta},a_g) = \hat{\mathcal{T}}_k \hat{Q}_k^t(s_g,F_{s_\Delta},a_g)$
    for $t\in\N$. Denote the fixed-point of $\hat{\mathcal{T}}_k$ by $\hat{Q}^*_k$ such that $\hat{\mathcal{T}}_k \hat{Q}^*_k(s_g,F_{s_\Delta},a_g) = \hat{Q}^*_k(s_g,F_{s_\Delta},a_g)$.\\
\end{definition}

\begin{definition}[$\hat{Q}_{k,m}^\est$]\label{defn:qkmest}Suppose $\hat{Q}_{k,m}^0:=0$ and let $\hat{Q}_{k,m}^{t+1}(s_g,F_{s_\Delta},a_g) = \hat{\mathcal{T}}_{k,m} \hat{Q}_{k,m}^t(s_g,F_{s_\Delta},a_g)$
    for $t\in\N$. Denote the fixed-point of $\hat{\mathcal{T}}_{k,m}$ by $\hat{Q}^\est_{k,m}$ such that $\hat{\mathcal{T}}_{k,m} \hat{Q}^\est_{k,m}(s_g,F_{s_\Delta},a_g) = \hat{Q}^\est_{k,m}(s_g,F_{s_\Delta},a_g)$.
\end{definition}

Furthermore, recall the assumption on our empirical approximation of $\hat{Q}^*_k$:

\textbf{\cref{assumption:qest_qhat_error}}. For all $k\in [n]$ and $m\in\N$, we assume that:
\[\|\hat{Q}_{k,m}^{\est} - \hat{Q}_k^*\|_\infty \leq \epsilon_{k,m}\]
\begin{corollary}\label{corollary:backprop}
    Observe that by backpropagating results of the $\gamma$-contractive property for $T$  time steps: \begin{equation}\|\hat{{Q}}_k^* - \hat{{Q}}_k^T\|_\infty \leq \gamma^T \cdot \|\hat{Q}_k^* - \hat{Q}_k^0\|_\infty
    \end{equation}
\begin{equation}\|\hat{{Q}}_{k,m}^\est - \hat{{Q}}_{k,m}^T\|_\infty \leq \gamma^T \cdot \|\hat{Q}_{k,m}^\est - \hat{Q}^0_{k,m}\|_\infty
    \end{equation}
    
Further, noting that $\hat{Q}_k^0 = \hat{Q}_{k,m}^0 := 0$, $\|\hat{Q}_k^*\|_\infty \leq \frac{\tilde{r}}{1-\gamma}$, and $\|\hat{Q}_{k,m}^\est\|_\infty \leq \frac{\tilde{r}}{1-\gamma}$ from \cref{lemma: Q-bound}:
\begin{equation}\label{eqn: qstar,qt decay}\|\hat{Q}_k^* - \hat{Q}_k^T\|_\infty \leq \gamma^T \frac{\tilde{r}}{1-\gamma}
\end{equation}
\begin{equation}\label{eqn: qest,qt decay}\|\hat{Q}_{k,m}^\est - \hat{Q}_{k,m}^T\|_\infty \leq \gamma^T \frac{\tilde{r}}{1-\gamma}
    \end{equation}
\end{corollary}

\begin{remark}
\emph{\cref{corollary:backprop} characterizes the error decay between $\hat{Q}_k^T$ and $\hat{Q}_k^*$ as well as between $\hat{Q}_{k,m}^T$ and $\hat{Q}_{k,m}^\est$ and shows that it decays exponentially in the number of corresponding Bellman iterations with the $\gamma^T$ multiplicative factor.}
\end{remark}

Furthermore, we characterize the maximal policies greedy policies obtained from $Q^*, \hat{Q}_k^*$, and $\hat{Q}_{k,m}^\est$.

\begin{definition}[$\pi^*$] The greedy policy derived from $Q^*$ is \[\pi^*(s) := \arg\max_{a_g\in\mathcal{A}_g} Q^*(s,a_g).\]
\end{definition}
\begin{definition}[$\hat{\pi}_k^*$]
    The greedy policy from $\hat{Q}_k^*$ is
\[\hat{\pi}_k^*(s_g,F_{s_\Delta}) := \arg\max_{a_g\in\mathcal{A}_g} \hat{Q}_k^*(s_g,F_{s_\Delta},a_g).\]
\end{definition}
\begin{definition}[$\hat{\pi}_{k,m}^\est$]
    The greedy policy from $\hat{Q}_{k,m}^\est$ is given by \[\hat{\pi}_{k,m}^\est(s_g,F_{s_\Delta}):=\arg\max_{a_g\in\mathcal{A}_g} \hat{Q}_{k,m}^\est(s_g,F_{s_\Delta},a_g).\]
\end{definition}

Figure \ref{figure: algorithm/analysis flow} details the analytic flow on how we use the empirical adapted Bellman operator to perform value iteration on $\hat{Q}_{k,m}$ to get $\hat{Q}_{k,m}^\est$ which approximates $Q^*$.
\begin{figure}[h]
\[ \begin{tikzcd}
\hat{Q}^0_{k,m}(s_g,F_{s_{\Delta}},a_g) \arrow[swap]{d}{\substack{(1)}}& & & &\\%
\hat{Q}_{k,m}^\est(s_g, F_{s_{\Delta}}, a_g) \arrow{r}{\substack{(2)\\=}} & \hat{Q}_k^*(s_g, F_{s_\Delta}, a_g)\arrow{r}{\substack{(3)\\ \approx}} & \hat{Q}_n^*(s_g, F_{s_{[n]}}, a_g)\arrow{d}{\substack{(4)\\=}}&\\
& & Q^*(s_g, s_{[n]}, a_g)
\end{tikzcd}
\]
\caption{Flow of the algorithm and relevant analyses in learning $Q^*$. Here, (1) follows by performing \cref{algorithm: approx-dense-tolerable-Q-learning} (\texttt{SUBSAMPLE-Q}: Learning) on $\hat{Q}_{k,m}^0$. (2) follows from \cref{assumption:qest_qhat_error}. (3) follows from the Lipschitz continuity and total variation distance bounds in \cref{thm:lip,thm:tvd}. Finally, (4) follows from noting that $\hat{Q}_n^* = Q^*$. 
}\label{figure: algorithm/analysis flow}
\end{figure}

\cref{algorithm: approx-dense-tolerable-Q-learning-stable} provides a stable implementation of \cref{algorithm: approx-dense-tolerable-Q-learning}: \texttt{SUBSAMPLE-Q}: Learning, where we incorporate a sequence of learning rates $\{\eta_t\}_{t\in[T]}$ into the framework \cite{Watkins_Dayan_1992}. \cref{algorithm: approx-dense-tolerable-Q-learning-stable} is also provably numerical stable under fixed-point arithmetic \cite{anand_et_al:LIPIcs.ICALP.2024.10}.

\begin{algorithm}[ht]
\caption{Stable (Practical) Implementation of \cref{algorithm: approx-dense-tolerable-Q-learning}:  \texttt{SUBSAMPLE-Q}: Learning }\label{algorithm: approx-dense-tolerable-Q-learning-stable}
\begin{algorithmic}[1]
\REQUIRE A multi-agent system as described in \cref{section: preliminaries}. Parameter $T$ for the number of iterations in the initial value iteration step. Hyperparameter $k \in [n]$. Discount parameter $\gamma\in (0,1)$. Oracle $\mathcal{O}$ to sample $s_g'\sim {P}_g(\cdot|s_g,a_g)$ and $s_i\sim {P}_l(\cdot|s_i,s_g)$ for all $i\in[n]$. Sequence of learning rates $\{\eta_t\}_{t\in [T]}$ where $\eta_t \in (0,1]$.\\
\STATE Choose any $\Delta\subseteq [n]$ such that $|\Delta|=k$.
\STATE Set $\hat{Q}^0_{k,m}(s_g, F_{s_\Delta}, a_g)=0$ for $s_g\in\mathcal{S}_g, F_{s_\Delta}\in\Theta_k^{|\mathcal{S}_l|}, a_g\in\mathcal{A}_g$.
\FOR{$t=1$ to $T$}
\FOR{$(s_g,F_{s_\Delta}) \in \mathcal{S}_g\times \Theta_k^{|\mathcal{S}_l|}$}
\FOR{$a_g \in \mathcal{A}_g$}
\STATE $\hat{Q}^{t+1}_{k,m}(s_g, F_{s_\Delta}, a_g) \gets (1 - \eta_t)\hat{Q}_{k,m}^{t}(s_g,F_{s_\Delta},a_g) + \eta_t \hat{\mathcal{T}}_{k,m} \hat{Q}^t_{k,m}(s_g, F_{s_\Delta}, a_g)$
\ENDFOR
\ENDFOR
\ENDFOR
\STATE For all $(s_g,F_{s_\Delta}) \in \mathcal{S}_g\times\Theta_k^{|\mathcal{S}_l|}$, let the approximate policy be \[\hat{\pi}_{k,m}^T(s_g, F_{s_\Delta}) = \arg\max_{a_g\in\mathcal{A}_g}\hat{Q}_{k,m}^T(s_g, F_{s_\Delta}, a_g).\]
\end{algorithmic}
\end{algorithm}

Notably, $\hat{Q}_{k,m}^t$ in \cref{algorithm: approx-dense-tolerable-Q-learning-stable} due to a similar $\gamma$-contractive property as in \cref{lemma: gamma-contraction of adapted Bellman operator}, given an appropriately conditioned sequence of learning rates $\eta_t$:
\begin{theorem}
    As $T\to\infty$, if $\sum_{t=1}^T \eta_t = \infty$, and $\sum_{t=1}^T \eta_t^2 < \infty$, then $Q$-learning converges to the optimal $Q$ function asymptotically with probability $1$.
\end{theorem}

Furthermore, finite-time guarantees with the learning rate and sample complexity have been shown recently in \cite{pmlr-v151-chen22i}, which when adapted to our $\hat{Q}_{k,m}$ framework in \cref{algorithm: approx-dense-tolerable-Q-learning-stable} yields:

\begin{theorem}[\cite{pmlr-v151-chen22i}] For all $t\in[T]$ and $\epsilon>0$, if $\eta_t = (1-\gamma)^4 \epsilon^2$ and $T = k^{|\mathcal{S}_l|}|\mathcal{S}_g||\mathcal{A}_g|/(1-\gamma)^5 \epsilon^2$,  
    \begin{align*}\|\hat{Q}_{k,m}^T - \hat{Q}_{k,m}^\est\|\leq \epsilon. \end{align*}
\end{theorem}

This global decision-making problem can be viewed as a generalization of the network setting to a specific type of dense graph: the star graph (Figure 4). We briefly elaborate more on this connection below.

\begin{definition}[Star Graph $S_n$]\label{definition: star graph} For $n\in\mathbb{N}$, the star graph $S_n$ is the complete bipartite graph $K_{1,n}$.
\end{definition}
$S_n$ captures the graph density notion by saturating the set of neighbors for the central node. Furthermore, it models interactions between agents identically to our setting, where the central node is a global agent and the peripheral nodes are local agents. The cardinality of the search space simplex for the optimal policy is $|\mathcal{S}_g||\mathcal{S}_l|^n|\mathcal{A}_g|$, which is exponential in $n$. Hence, this problem cannot be naively modeled by an MDP: we need to exploit the symmetry of the local agents. This intuition allows our subsampling algorithm to run in polylogarithmic time (in $n$). Further, works that leverage the exponential decaying property that truncates the search space for policies over immediate neighborhoods of agents still rely on the assumption that the graph neighborhood for the agent is sparse \cite{DBLP:conf/nips/LinQHW21,10.5555/3495724.3495899,pmlr-v120-qu20a,DBLP:journals/corr/abs-2006-06555}; however, the graph $S_n$  violates this local sparsity condition; hence, previous methods do not apply to this problem instance.
\begin{center}
\begin{figure}[h]
\begin{center}
\begin{tikzpicture}%
  [>=stealth,
   shorten >=1pt,
   node distance=1.4cm,
   on grid,
   auto,
   every state/.style={draw=black!60, fill=black!5, very thick}
  ]
\node[state] (1) {$1$};
\node[state] (2) [right=of 1] {$2$};
\node[state] (0) [above right=of 2] {$0$};
\node[state] (3) [below =of 0] {$3$};
\node (5) [below  right=of 0] {$\dots$};
\node[state] (4) [right=of 5] {$n$};

\path[-]
   (0)
   edge node {} (1)
   edge node {} (2)
   edge node {} (3)
   edge node {} (5)
   edge node {} (4);
\end{tikzpicture}
\caption{Star graph $S_n$}
\end{center}
\end{figure}
\end{center}

\section{Proof of Lipschitz-Continuity Bound}
This section proves the Lipschitz-continuity bound \cref{thm:lip} between $\hat{Q}_k^*$ and $Q^*$ in \cref{lemma: Q lipschitz continuous wrt Fsdelta}  and includes a framework to compare $\frac{1}{\binom{n}{k}}\sum_{\Delta\in\binom{[n]}{k}}\hat{Q}^*_k(s_g,F_{s_\Delta},a_g)$ and $Q^*(s,a_g)$ in \cref{lemma:comparing_expectation}. The following definition will be relevant to the proof of \cref{thm:lip}.

\begin{definition}\label{definition: joint_transition_probability}[Joint Stochastic Kernels] The joint stochastic kernel on $(s_g,s_\Delta)$ for $\Delta\subseteq [n]$ where $|\Delta|=k$ is defined as $\mathcal{J}_k:\mathcal{S}_g\times\mathcal{S}_l^{k}\times\mathcal{S}_g\times\mathcal{A}_g\times\mathcal{S}_l^{k}\to[0,1]$, where \begin{equation}\mathcal{J}_k(s_g',s_\Delta'|s_g, a_g, s_\Delta) := \Pr[(s_g',s_\Delta')|s_g, a_g, s_\Delta]\end{equation}
\end{definition}

\begin{theorem}[$\hat{Q}_k^T$ is $(\sum_{t=0}^{T-1} 2\gamma^t)\|r_l(\cdot,\cdot)\|_\infty$-Lipschitz continuous with respect to $F_{s_\Delta}$ in total variation distance]\label{lemma: Q lipschitz continuous wrt Fsdelta}
Suppose $\Delta,\Delta'\subseteq[n]$ such that $|\Delta|=k$ and $|\Delta'|=k'$. Then:
\[\left|\hat{Q}^T_k(s_g,F_{s_\Delta},a_g) - \hat{Q}^T_{k'}(s_g, F_{s_{\Delta'}}, a_g)\right| \leq \left(\sum_{t=0}^{T-1} 2\gamma^t\right)\|r_l(\cdot,\cdot)\|_\infty \cdot \mathrm{TV}\left(F_{s_\Delta}, F_{s_{\Delta'}}\right)\]
\end{theorem}
\begin{proof} 
We prove this inductively. Note that  $\hat{Q}_k^0(\cdot,\cdot,\cdot) = \hat{Q}_{k'}^0(\cdot,\cdot,\cdot)=0$ from the initialization step in \cref{algorithm: approx-dense-tolerable-Q-learning}, which proves the lemma for $T=0$ since $\mathrm{TV}(\cdot,\cdot)\geq 0$. For the remainder of this proof, we adopt the shorthand $\E_{s_g',s_\Delta'}$ to refer to $\E_{s_g'\sim P_g(\cdot|s_g, a_g), s_i'\sim P_l(\cdot|s_i,s_g),\forall i\in\Delta}$. \\

Then, at $T=1$:
\begin{align*}
|\hat{Q}_k^1(s_g,F_{s_\Delta},a_g)&-\hat{Q}_{k'}^1(s_g, F_{s_{\Delta'}}, a_g)| \\
&= \left|\hat{\mathcal{T}}_k\hat{Q}_k^0(s_g,F_{s_\Delta},a_g)-\hat{\mathcal{T}}_{k'}\hat{Q}_{k'}^0(s_g, F_{s_{\Delta'}}, a_g)\right| \\
&= |r(s_g, F_{s_\Delta}, a_g)+\gamma \E_{s_g', s'_\Delta}\max_{a_g'\in\mathcal{A}_g} \hat{Q}_k^0(s_g',F_{s'_\Delta},a'_g)\\
&-r(s_g,F_{s_{\Delta'}},a_g) -\gamma \E_{s_g', s'_{\Delta'}} \max_{a_g'\in\mathcal{A}_g}\hat{Q}_{k'}^0(s'_g, F_{s'_{\Delta'}}, a'_g)| \\
&= |r(s_g, F_{s_\Delta}, a_g)-r(s_g, F_{s_{\Delta'}},a_g)| \\
&= \left|\frac{1}{k}\sum_{i\in\Delta}r_l(s_g,s_i) -\frac{1}{k'}\sum_{i\in\Delta'}r_l(s_g,s_i)\right| \\
&= |\E_{s_l \sim F_{s_\Delta}}r_l(s_g, s_l) - \E_{s_l' \sim F_{s_{\Delta'}}}r_l(s_g, s_l')|
\end{align*}
In the first and second equalities, we use the time evolution property of $\hat{Q}_k^1$ and $\hat{Q}_{k'}^1$ by applying the adapted Bellman operators $\hat{\mathcal{T}}_k$ and $\hat{\mathcal{T}}_{k'}$ to $\hat{Q}_k^0$ and $\hat{Q}_{k'}^0$, respectively, and expanding. In the third and fourth equalities, we note that $\hat{Q}_k^0(\cdot,\cdot,\cdot)=\hat{Q}_{k'}^0(\cdot,\cdot,\cdot)=0$, and subtract the common `global component' of the reward function. 

Then, noting the general property that for any function $f:\mathcal{X}\to\mathcal{Y}$ for $|\mathcal{X}|<\infty$ we can write $f(x) = \sum_{y\in\mathcal{X}}f(y)\mathbbm{1}\{y=x\}$, we have:
\begin{align*}
|\hat{Q}^1_k(s_g,F_{s_\Delta},a_g) &- \hat{Q}_{k'}^1(s_g, F_{s_{\Delta'}}, a_g)| \\
&= \left|\E_{s_l \sim F_{s_\Delta}}\left[\sum_{z\in \mathcal{S}_l}r_l(s_g, z)\mathbbm{1}\{s_l = z\}\right] - \E_{s_l' \sim F_{s_{\Delta'}}}\left[\sum_{z\in\mathcal{S}_l}r_l(s_g, z)\mathbbm{1}\{s_l' = z\}\right]\right| \\
&= |\sum_{z\in\mathcal{S}_l}r_l(s_g,z)\cdot (\E_{s_l\sim F_{s_\Delta}}\mathbbm{1}\{s_l=z\} - \E_{s_l'\sim F_{s_{\Delta'}}}\mathbbm{1}\{s_l'=z\})| \\
&= |\sum_{z\in\mathcal{S}_l}r_l(s_g,z)\cdot (F_{s_\Delta}(z) - F_{s_{\Delta'}}(z))|  \\
&\leq |\max_{z\in\mathcal{S}_l}r_l(s_g,z)|\cdot \sum_{z\in\mathcal{S}_l}|F_{s_\Delta}(z) - F_{s_{\Delta'}}(z)|\\
&\leq 2\|r_l(\cdot,\cdot)\|_\infty \cdot \mathrm{TV}(F_{s_\Delta}, F_{s_{\Delta'}})
\end{align*}
The second equality follows from the linearity of expectations, and the third equality follows by noting that for any random variable $X\sim\mathcal{X}$, $\E_X\mathbbm{1}[X=x]=\Pr[X=x]$.
Then, the first inequality follows from an application of the triangle inequality and the Cauchy-Schwarz inequality, and the second inequality follows by the definition of total variation distance. Thus, when $T=1$, $\hat{Q}$ is $(2\|r_l(\cdot,\cdot)\|_\infty)$-Lipschitz continuous with respect to total variation distance, proving the base case. 

We now assume that for $T\leq t'\in\N$:
\begin{align*}\left|\hat{Q}^{T}_k(s_g,F_{s_\Delta}, a_g) - \hat{Q}_{k'}^{T}(s_g,F_{s_{\Delta'}}, a_g)\right| \leq \left(\sum_{t=0}^{T-1}2\gamma^t\right)\|r_l(\cdot,\cdot)\|_\infty\cdot\mathrm{TV}\left(F_{s_\Delta}, F_{s_{\Delta'}}\right)\end{align*}

Then, inductively we have:
\begin{align*}
    |\hat{Q}_k^{T+1}(s_g,F_{s_\Delta}, a_g) &- \hat{Q}_{k'}^{T+1}(s_g,F_{s_{\Delta'}}, a_g)|\\
&\leq\left|\frac{1}{|\Delta|}\sum_{i\in\Delta}r_l(s_g,s_i)-\frac{1}{|\Delta'|}\sum_{i\in\Delta'}r_l(s_g,s_i)\right|\\
&\quad\quad+\gamma\left|\E_{s_g', s'_{\Delta}}\max_{a_g'\in\mathcal{A}_g}\hat{Q}_k^{T}(s_g',F_{s_\Delta'}, a_g')-\E_{s_g', s'_{\Delta'}}\max_{a_g'\in\mathcal{A}_g}\hat{Q}_{k'}^{T}(s_g',F_{s_{\Delta'}'}, a_g')\right| \\
&\leq 2\|r_l(\cdot,\cdot)\|_\infty\cdot\mathrm{TV}\left(F_{s_\Delta}, F_{s_{\Delta'}}\right) \\
&\quad\quad+ \gamma\left|\E_{s_g', s'_{\Delta}}\max_{a_g'\in\mathcal{A}_g}\hat{Q}_k^{T}(s_g',F_{s_\Delta'}, a_g')-\E_{s_g', s'_{\Delta'}}\max_{a_g'\in\mathcal{A}_g}\hat{Q}_{k'}^{T}(s_g',F_{s_{\Delta'}'}, a_g')\right|
\end{align*}
\noindent In the first equality, we use the time evolution property of $\hat{Q}_k^{T+1}$ and $\hat{Q}_{k'}^{T+1}$ by applying the adapted-Bellman operators $\hat{\mathcal{T}}_k$ and $\hat{\mathcal{T}}_{k'}$ to $\hat{Q}_k^T$ and $\hat{Q}_{k'}^T$, respectively. We then expand and use the triangle inequality. In the first term of the second inequality, we use our Lipschitz bound from the base case. For the second term, we now rewrite the expectation over the states $s_g', s_\Delta', s_{\Delta'}'$ into an expectation over the joint transition probabilities $\mathcal{J}_k$ and $\mathcal{J}_{k'}$ from \cref{definition: joint_transition_probability}. 

Therefore, using the shorthand $\E_{(s_g',s_\Delta')\sim\mathcal{J}_k}$ to denote $\E_{(s_g',s_\Delta')\sim\mathcal{J}_k(\cdot,\cdot|s_g,a_g,s_\Delta)}$, we have: 
\begin{align*}
    &|\hat{Q}_k^{T+1}(s_g,F_{s_\Delta}, a_g) - \hat{Q}_{k'}^{T+1}(s_g,F_{s_{\Delta'}}, a_g)| \\
    &\leq 2\|r_l(\cdot,\cdot)\|_\infty\cdot\mathrm{TV}(F_{s_\Delta}, F_{s_{\Delta'}}) \\
    &\quad\quad\quad\quad\quad+ \gamma|\E_{(s_g', s'_{\Delta})\sim\mathcal{J}_k}\max_{a_g'\in\mathcal{A}_g}\!\hat{Q}_k^{T}(s_g',F_{s_\Delta'}, a_g')\!-\!\E_{(s_g', s'_{\Delta'})\sim\mathcal{J}_{k'}}\max_{a_g'\in\mathcal{A}_g}\!\hat{Q}_{k'}^{T}(s_g',F_{s_{\Delta'}'}, a_g')|\\
    &\leq 2\|r_l(\cdot,\cdot)\|_\infty\cdot\mathrm{TV}(F_{s_\Delta}, F_{s_{\Delta'}}) \\
    &\quad\quad\quad\quad\quad+ \gamma\max_{a_g'\in\mathcal{A}_g}|\E_{(s_g', s'_{\Delta})\sim\mathcal{J}_k} \hat{Q}_k^{T}(s_g', F_{s'_\Delta}, a_g') - \E_{(s_g', s'_{\Delta'})\sim\mathcal{J}_{k'}} \hat{Q}_{k'}^{T}(s_g', F_{s'_{\Delta'}}, a_g')| \\
&\leq 2\|r_l(\cdot,\cdot)\|_\infty\cdot\mathrm{TV}(F_{s_\Delta}, F_{s_{\Delta'}}) + \gamma\left(\sum_{\tau=0}^{T-1}2\gamma^\tau\right)\|r_l(\cdot,\cdot)\|_\infty\cdot\mathrm{TV}(F_{s_\Delta},  F_{s_{\Delta'}})\\
&=\left(\sum_{\tau=0}^{T}2\gamma^\tau\right) \|r_l(\cdot,\cdot)\|_\infty\cdot\mathrm{TV}(F_{s_\Delta}, F_{s_{\Delta'}})
\end{align*}

\noindent In the first inequality, we rewrite the expectations over the states as the expectation over the joint transition probabilities. The second inequality then follows from \cref{lemma: expectation_expectation_max_swap}. To apply it to \cref{lemma: expectation_expectation_max_swap}, we superficially conflate the joint expectation over $(s_g, s_{\Delta\cup\Delta'})$ and reduce it back to the original form of its expectation. Finally, the third inequality follows from \cref{lemma: expectation Q lipschitz continuous wrt Fsdelta}.

Then, by the inductive hypothesis, the claim is proven.\qedhere
\end{proof}

\newpage

\begin{lemma}     \label{lemma: expectation Q lipschitz continuous wrt Fsdelta}For all $T\in\N$, for any $a_g, a_g'\in\mathcal{A}_g,s_g\in\mathcal{S}_g, s_\Delta\in\mathcal{S}_l^{k}$, and for all joint stochastic kernels $\mathcal{J}_k$ as defined in \cref{definition: joint_transition_probability}, we have that $\E_{(s_g',s_\Delta')\sim\mathcal{J}_k(\cdot,\cdot|s_g,a_g,s_\Delta)} \hat{Q}_k^T(s_g', F_{s'_\Delta}, a_g')$ is $(\sum_{t=0}^{t-1})2\gamma^t)\|r_l(\cdot,\cdot)\|_\infty)$-Lipschitz continuous with respect to $F_{s_\Delta}$ in total variation distance:
\begin{align*}|\E_{(s_g',s_\Delta')\sim\mathcal{J}_k(\cdot,\cdot|s_g,a_g,s_\Delta)}\hat{Q}_k^T(s'_g,F_{s'_\Delta},a_g') &- \E_{(s_g',s_{\Delta'}')\sim\mathcal{J}_{k'}(\cdot,\cdot|s_g,a_g,s_{\Delta'})}\hat{Q}_{k'}^T(s'_g, F_{s'_{\Delta'}}, a_g')| \\
&\quad\quad\leq \left(\sum_{\tau=0}^{T-1}2\gamma^\tau\right)\|r_l(\cdot,\cdot)\|_\infty\cdot\mathrm{TV}\left(F_{s_\Delta}, F_{s_{\Delta'}}\right)
\end{align*}
\end{lemma}
\begin{proof}
We prove this inductively. At $T=0$, the statement is true since $\hat{Q}_k^0(\cdot,\cdot,\cdot) = \hat{Q}_{k'}^0(\cdot,\cdot,\cdot) = 0$ and $\mathrm{TV}(\cdot,\cdot)\geq 0$. For $T=1$, applying the adapted Bellman operator yields: 
\begin{align*}
&|\E_{(s_g',s_\Delta')\sim\mathcal{J}_k(\cdot,\cdot|s_g,a_g,s_\Delta)}\hat{Q}_k^1(s'_g,F_{s'_\Delta},a_g') - \E_{(s_g',s_{\Delta'}')\sim\mathcal{J}_{k'}(\cdot,\cdot|s_g,a_g,s_{\Delta'})}\hat{Q}_{k'}^1(s'_g, F_{s'_{\Delta'}}, a_g')| \\
&\quad\quad\quad\quad=\left| \E_{(s_g',s_{\Delta\cup\Delta'}')\sim\mathcal{J}_{|\Delta\cup\Delta'|}(\cdot,\cdot|s_g, a_g,s_{\Delta\cup\Delta'})}\left[\frac{1}{|\Delta|}\sum_{i\in\Delta}r_l(s'_g,s'_i) - \frac{1}{|\Delta'|}\sum_{i\in\Delta'}r_l(s_g',s_i')\right]\right| \\
&\quad\quad\quad\quad= \left|\E_{(s_g',s'_{\Delta\cup\Delta'})\sim\mathcal{J}_{|\Delta\cup\Delta'|}(\cdot,\cdot|s_g,a_g,s_{\Delta\cup\Delta'})}\left[\sum_{z\in\mathcal{S}_l}r_l(s_g',z)\cdot (F_{s'_\Delta}(z) -  F_{s'_{\Delta'}}(z))\right]\right|
\end{align*}
Similarly to \cref{lemma: Q lipschitz continuous wrt Fsdelta}, we implicitly write the result as an expectation over the reward functions and use the general property that for any function $f:\mathcal{X}\to\mathcal{Y}$ for $|\mathcal{X}|<\infty$, we can write $f(x) = \sum_{y\in\mathcal{X}}f(y)\mathbbm{1}\{y=x\}$. Then, taking the expectation over the indicator variable yields the second equality. As a shorthand, let $\mathfrak{D}$ denote the distribution of $s_{g}'\sim \sum_{s'_{\Delta\cup\Delta'}\in\mathcal{S}_l^{|\Delta\cup\Delta'|}}\mathcal{J}_{|\Delta\cup\Delta|}(\cdot,s'_{\Delta\cup\Delta'}|s_g,a_g,s_{\Delta\cup\Delta'})$. Then, by the law of total expectation,
\begin{align*}
&|\E_{(s_g',s_\Delta')\sim\mathcal{J}_k(\cdot,\cdot|s_g,a_g,s_\Delta)}\hat{Q}_k^1(s'_g,F_{s'_\Delta},a_g') - \E_{(s_g',s_{\Delta'}')\sim\mathcal{J}_{k'}(\cdot,\cdot|s_g,a_g,s_{\Delta'})}\hat{Q}_{k'}^1(s'_g, F_{s'_{\Delta'}}, a_g')| \\
&= |\E_{s_g'\sim \mathfrak{D}} \sum_{z\in\mathcal{S}_l}r_l(s_g',z)\E_{s'_{\Delta\cup\Delta'}\sim\mathcal{J}_{|\Delta\cup\Delta'|}(\cdot|s_g',s_g,a_g,s_{\Delta\cup\Delta'})}(F_{s'_\Delta}(z) - F_{s'_{\Delta'}}(z))|\\
&\leq  \|r_l(\cdot,\cdot)\|_\infty\cdot\E_{s_g'\sim\mathfrak{D}}\sum_{z\in\mathcal{S}_l}|\E_{s'_{\Delta\cup\Delta'}\sim\mathcal{J}_{|\Delta\cup\Delta'|}(\cdot|s_g',s_g, a_g,s_{\Delta\cup\Delta'})}(F_{s'_\Delta}(z)-F_{s'_{\Delta'}}(z))|\\
&\leq 2\|r_l(\cdot,\cdot)\|_\infty\cdot \E_{s_g'\sim\mathfrak{D}} \mathrm{TV}(\E_{s'_{\Delta\cup\Delta'}|s_g'}F_{s'_\Delta}, \E_{s'_{\Delta\cup\Delta'}|s_g'}F_{s'_{\Delta'}}) \\
&\leq 2\|r_l(\cdot,\cdot)\|_\infty \!\cdot \!\mathrm{TV}(F_{s_\Delta}, F_{s_{\Delta'}}) 
\end{align*}
In the ensuing inequalities, we first use Jensen's inequality and the triangle inequality to pull out $\E_{s_g'}\sum_{z\in\mathcal{S}_l}$ from the absolute value, and then use Cauchy-Schwarz to further factor $\|r_l(\cdot,\cdot)\|_\infty$. The second inequality follows from \cref{lemma: generalized tvd linear bound} and does not have a dependence on $s_g'$ thus eliminating $\E_{s_g'}$ and proving the base case. \\

\noindent We now assume that for $T\leq t'\in\N$, for all joint stochastic kernels $\mathcal{J}_k$ and $\mathcal{J}_{k'}$, and for all $a_g'\in \mathcal{A}_g$:
\begin{align*}
|\E_{(s_g',s_\Delta')\sim\mathcal{J}_k(\cdot,\cdot|s_g,a_g,s_\Delta)}\hat{Q}_k^T(s'_g,F_{s'_\Delta},a_g') &- \E_{(s_g',s_{\Delta'}')\sim\mathcal{J}_{k'}(\cdot,\cdot|s_g,a_g,s_{\Delta'})}\hat{Q}_{k'}^T(s'_g, F_{s'_{\Delta'}}, a_g')| \\
&\leq \left(\sum_{t=0}^{T-1} 2\gamma^t\right)\|r_l(\cdot,\cdot)\|_\infty \cdot \mathrm{TV}(F_{s_\Delta}, F_{s_{\Delta'}})
\end{align*}
For the remainder of the proof, we adopt the shorthand $\E_{(s_g',s_\Delta')\sim{\mathcal{J}}}$ to denote $\E_{(s_g',s_\Delta')\sim\mathcal{J}_{|\Delta|}(\cdot,\cdot|s_g,a_g,s_\Delta)}$, and $\E_{(s_g'',s_\Delta'')\sim{\mathcal{J}}}$ to denote $\E_{(s_g'',s_\Delta'')\sim\mathcal{J}_{|\Delta|}(\cdot,\cdot|s_g',a_g',s_\Delta')}$. \\

Then, inductively, we have:
\begin{align*}
&|\E_{(s_g',s_\Delta')\sim{\mathcal{J}}} \hat{Q}_k^{T+1}(s'_g,F_{s'_\Delta},a_g') -\E_{(s_g',s_{\Delta'}')\sim{\mathcal{J}}}\hat{Q}_{k'}^{T+1}(s'_g, F_{s'_{\Delta'}}, a_g')| \\ &=|\E_{(s_g',s_{\Delta\cup\Delta'}')\sim\mathcal{J}}[r(s_g',s'_\Delta,a_g')-r(s_g',s'_{\Delta'},a_g') \\
 &\quad\quad+\gamma\E_{(s_g'',s_{\Delta\cup\Delta'}'')\sim{\mathcal{J}}}[\max_{a_g''\in\mathcal{A}_g}\! \hat{Q}_k^T(s_g'',F_{s_\Delta''},a_g'')-\max_{a_g''\in\mathcal{A}_g} \hat{Q}_{k'}^T(s_g'',F_{s_{\Delta'}''}, a_g'')]]| \\
&\leq 2\|r_l(\cdot,\cdot)\|_\infty \cdot \mathrm{TV}(F_{s_\Delta}, F_{s_{\Delta'}}) \\
&\quad\quad+ \gamma|\E_{(s_g',s_{\Delta\cup\Delta'}')\sim\mathcal{J}}[\E_{(s_g'',s_{\Delta\cup\Delta'}'')\sim{\mathcal{J}}}[\max_{a_g''\in\mathcal{A}_g} \hat{Q}_k^T(s_g'',F_{s_\Delta''},a_g'') -\max_{a_g''\in\mathcal{A}_g} \hat{Q}_{k'}^T(s_g'',F_{s_{\Delta'}''}, a_g'')]]|
\end{align*}

\noindent Here, we expand out $\hat{Q}_k^{T+1}$ and $\hat{Q}_{k'}^{T+1}$ using the adapted Bellman operator. In the ensuing inequality, we apply the triangle inequality and bound the first term using the base case. Then, note that \[\E_{(s_g',s_{\Delta\cup\Delta'}')\sim\mathcal{J}(\cdot,\cdot|s_g, a_g, s_{\Delta\cup\Delta'})} \E_{(s_g'',s_{\Delta\cup\Delta'}'')\sim\mathcal{J}(\cdot,\cdot|s_g', a_g', s'_{\Delta\cup\Delta'})}\max_{a_g''\in\mathcal{A}_g} \hat{Q}_k^T(s_g'',F_{s_\Delta''},a_g'')\]
is, for some stochastic function $\mathcal{J}_{|\Delta\cup\Delta'|}'$, equal to \[\E_{(s_g'', s_{\Delta\cup\Delta'}'')\sim\mathcal{J}_{|\Delta\cup\Delta'|}'(\cdot,\cdot|s_g, a_g, s_{\Delta\cup\Delta'})}\max_{a_g''\in\mathcal{A}_g}\hat{Q}_k^T(s_g'', F_{s_\Delta''}, a_g''),\] where $\mathcal{J}'$ is implicitly a function of $a_g'$ which is fixed from the beginning. 

In the special case where $a_g = a_g'$, we can derive an explicit form of $\mathcal{J}'$ which we show in \cref{lemma: combining transition probabilities}. As a shorthand, we denote $\E_{(s_g'',s_{\Delta\cup\Delta'}'')\sim\mathcal{J}_{|\Delta\cup\Delta'|}'(\cdot,\cdot|s_g, a_g, s_{\Delta\cup\Delta'})}$ by $\E_{(s_g'',s_{\Delta\cup\Delta'}'')\sim\mathcal{J}'}$. 

Therefore,
\begin{align*}
|\E_{(s_g',s_\Delta')\sim{\mathcal{J}}}& \hat{Q}_k^{T+1}(s'_g,F_{s'_\Delta},a_g') -\E_{(s_g',s_{\Delta'}')\sim{\mathcal{J}}}\hat{Q}_{k'}^{T+1}(s'_g, F_{s'_{\Delta'}}, a_g')| \\ 
&\leq 2\|r_l(\cdot,\cdot)\|_\infty \cdot \mathrm{TV}(F_{s_\Delta}, F_{s_{\Delta'}}) + \gamma|\E_{(s_g'',s_{\Delta\cup\Delta'}'')\sim{\mathcal{J}'}}\max_{a_g''\in\mathcal{A}_g} \hat{Q}_k^T(s_g'',F_{s_\Delta''},a_g'') \\
&\quad\quad\quad\quad-\E_{(s_g'',s_{\Delta\cup\Delta'}'')\sim{\mathcal{J}'}}\max_{a_g''\in\mathcal{A}_g} \hat{Q}_{k'}^T(s_g'',F_{s_{\Delta'}''}, a_g'')| \\
&\leq 2\|r_l(\cdot,\cdot)\|_\infty\cdot \mathrm{TV}(F_{s_\Delta}, F_{s_{\Delta'}}) + \gamma\max_{a_g''\in\mathcal{A}_g}|\E_{(s_g'',s_{\Delta\cup\Delta'}'')\sim{\mathcal{J}'}} \hat{Q}^T_k(s_g'',F_{s_\Delta''},a_g'') \\
&\quad\quad\quad\quad- \E_{(s_g'',s_{\Delta\cup\Delta'}'')\sim{\mathcal{J}'}} \hat{Q}^T_{k'}(s_g'',F_{s_{\Delta'}''},a_g'')| \\ 
&\leq 2\|r_l(\cdot,\cdot)\|_\infty\cdot\mathrm{TV}(F_{s_\Delta}, F_{s_{\Delta'}}) + \gamma\left(\sum_{t=0}^{T-1}2\gamma^t\right)\|r_l(\cdot,\cdot)\|_\infty\cdot \mathrm{TV}(F_{s_\Delta}, F_{s_{\Delta'}}) \\
&=  \left(\sum_{t=0}^{T}2\gamma^t\right)\|r_l(\cdot,\cdot)\|_\infty \cdot \mathrm{TV}(F_{s_\Delta},F_{s_{\Delta'}})
\end{align*}
\noindent  The second inequality follows from \cref{lemma: expectation_expectation_max_swap} where we set the joint stochastic kernel to be $\mathcal{J}_{|\Delta\cup\Delta'|}'$. In the ensuing lines, we concentrate the expectation towards the relevant terms and use the induction assumption for the transition probability functions $\mathcal{J}'_k$ and $\mathcal{J}'_{k'}$. This proves the lemma. \qedhere
\end{proof}

\begin{remark}
\emph{Given a joint transition probability function $\mathcal{J}_{|\Delta\cup\Delta'|}$ as defined in \cref{definition: joint_transition_probability}, we can recover the transition function for a single agent $i\in \Delta\cup\Delta'$ given by $\mathcal{J}_1$ using the law of total probability and the conditional independence between $s_i$ and $s_g\cup s_{[n]\setminus i}$ in \cref{equation: lotp/condind}. This characterization is crucial in  \cref{lemma: generalized tvd linear bound} and \cref{lemma: expected next empirical distribution linearity}. }
    \begin{equation}
        \label{equation: lotp/condind}\mathcal{J}_1(\cdot|s_g',s_g, a_g, s_i) = \sum_{s'_{\Delta\cup\Delta'\setminus i}\sim \mathcal{S}_l^{|\Delta\cup\Delta'|-1}}\mathcal{J}_{|\Delta\cup\Delta'|}(s'_{\Delta\cup\Delta'\setminus i}, s'_i|s_g',s_g,a_g,s_{\Delta\cup\Delta'})
    \end{equation}
\end{remark}

\begin{lemma}\label{lemma: generalized tvd linear bound} Given a joint transition probability $\mathcal{J}_{|\Delta\cup\Delta'|}$ as defined in \cref{definition: joint_transition_probability}, 
\[\mathrm{TV}(\E_{s'_{\Delta\cup\Delta'}\sim\mathcal{J}_{|\Delta\cup\Delta'|}(\cdot|s_g', s_g, a_g, s_{\Delta\cup\Delta'})} F_{s'_{\Delta}}, \E_{s'_{\Delta\cup\Delta'}\sim\mathcal{J}_{|\Delta\cup\Delta'|}(\cdot|s_g', s_g, a_g, s_{\Delta\cup\Delta'})} F_{s'_{\Delta'}})\leq \mathrm{TV}(F_{s_\Delta}, F_{s_{\Delta'}})\]
\end{lemma}
\begin{proof}
Note that from \cref{lemma: expected next empirical distribution linearity}:
\begin{align*}\E_{s_{\Delta\cup\Delta'}'\sim\mathcal{J}_{|\Delta\cup\Delta'|}(\cdot,\cdot|s_g', s_g, a_g, s_{\Delta\cup\Delta'})}F_{s'_\Delta} 
&= \E_{s_{\Delta}'\sim\mathcal{J}_{|\Delta|}(\cdot,\cdot|s_g', s_g, a_g, s_{\Delta})} F_{s'_\Delta} \\
&= \mathcal{J}_1(\cdot|s_g(t+1),s_g(t),a_g(t),\cdot)F_{s_\Delta}
\end{align*}

\noindent Then, by expanding the TV distance in $\ell_1$-norm:
    \begin{align*}
        \mathrm{TV}&(\E_{s'_{\Delta\cup\Delta'}\sim\mathcal{J}_{|\Delta\cup\Delta'|}(\cdot|s_g', s_g, a_g, s_{\Delta\cup\Delta'})} F_{s'_{\Delta}}, \E_{s'_{\Delta\cup\Delta'}\sim\mathcal{J}_{|\Delta\cup\Delta'|}(\cdot|s_g', s_g, a_g, s_{\Delta\cup\Delta'})} F_{s'_{\Delta'}}) \\
        &\quad\quad\quad= \frac{1}{2}\|\mathcal{J}_1(\cdot|s_g(t+1),s_g(t),a_g(t),\cdot)F_{s_\Delta}\!-\!\mathcal{J}_1(\cdot|s_g(t+1),s_g(t),a_g(t),\cdot)F_{s_{\Delta'}}\|_1 \\
        &\quad\quad\quad\leq \|\mathcal{J}_1(\cdot|s_g(t+1),s_g(t),a_g(t),\cdot)\|_1 \cdot \frac{1}{2}\|F_{s_\Delta} \!-\! F_{s_{\Delta'}}\|_1 \\
        &\quad\quad\quad\leq \frac{1}{2}\|F_{s_\Delta} \!-\! F_{s_{\Delta'}}\|_1 \\
        &\quad\quad\quad= \mathrm{TV}(F_{s_\Delta},F_{s_{\Delta'}})
    \end{align*}
In the first inequality, we factorize $\|\mathcal{J}_1(\cdot|s_g(t+1),s_g(t), a_g(t))\|_1$ from the $\ell_1$-normed expression by the sub-multiplicativity of the matrix norm. Finally, since $\mathcal{J}_1$ is a column-stochastic matrix, we bound its norm by $1$ to recover the total variation distance between $F_{s_\Delta}$ and $F_{s_{\Delta'}}$. \qedhere \\
\end{proof}

\begin{lemma}\label{lemma: expected next empirical distribution linearity}Given the joint transition probability $\mathcal{J}_k$ from  \cref{definition: joint_transition_probability}:
\[\E_{s_{\Delta\cup\Delta'}(t+1)\sim\mathcal{J}_{|\Delta\cup\Delta'|}(\cdot|s_g(t+1), s_g(t), a_g(t), s_{\Delta\cup\Delta'}(t))} F_{s_\Delta (t+1)} :=  \mathcal{J}_{1}(\cdot|s_g(t+1),s_g(t), a_g(t), \cdot) F_{s_\Delta}(t)\]
\end{lemma}
\begin{proof}
First, observe that for all $x\in\mathcal{S}_l$:
\begin{align*}
&\E_{s_{\Delta\cup\Delta'}(t+1)\sim\mathcal{J}_{|\Delta\cup\Delta'|}(\cdot|s_g(t+1), s_g(t), a_g(t), s_{\Delta\cup\Delta'}(t))}F_{s_\Delta (t+1)}(x) \\
&\quad\quad\quad\quad =\frac{1}{|\Delta|}\sum_{i\in\Delta}\E_{s_{\Delta\cup\Delta'}(t+1)\sim\mathcal{J}_{|\Delta\cup\Delta'|}(\cdot|s_g(t+1), s_g(t), a_g(t), s_{\Delta\cup\Delta'}(t))}\mathbbm{1}(s_i({t+1})=x)\\
    &\quad\quad\quad\quad= \frac{1}{|\Delta|}\sum_{i\in\Delta}\Pr[s_i(t+1) = x | s_g(t+1), s_g(t), a_g(t), s_{\Delta\cup\Delta'}(t))] \\
    &\quad\quad\quad\quad= \frac{1}{|\Delta|}\sum_{i\in\Delta}\Pr[s_i(t+1) = x | s_g(t+1), s_g(t), a_g(t), s_{i}(t))] \\
    &\quad\quad\quad\quad= \frac{1}{|\Delta|}\sum_{i\in\Delta}\mathcal{J}_{1}(x|s_g(t+1), s_g(t), a_g(t), s_{i}(t))
\end{align*}
In the first line, we expand on the definition of $F_{s_\Delta(t+1)}(x)$. Finally, we note that $s_i(t+1)$ is conditionally independent to $s_{\Delta\cup\Delta'\setminus i}$, which yields the equality above. Then, aggregating across every entry $x\in\mathcal{S}_l$,
\begin{align*}
    &\E_{s_{\Delta\cup\Delta'}(t+1)\sim\mathcal{J}_{|\Delta\cup\Delta'|}(\cdot|s_g(t+1), s_g(t), a_g(t), s_{\Delta\cup\Delta'}(t))} F_{s_\Delta (t+1)} \\
    &\quad\quad\quad\quad= \frac{1}{|\Delta|}\sum_{i\in\Delta}\mathcal{J}_{1}(\cdot|s_g(t+1),s_g(t), a_g(t), \cdot) \vec{\mathbbm{1}}_{s_i(t)} \\
    &\quad\quad\quad\quad= \mathcal{J}_{1}(\cdot|s_g(t+1),s_g(t), a_g(t), \cdot) F_{s_\Delta}
\end{align*}
Notably, every $x$ corresponds to a choice of rows in $\mathcal{J}_{1}(\cdot|s_g(t+1),s_g(t), a_g(t), \cdot)$ and every choice of $s_i(t)$ corresponds to a choice of columns in $\mathcal{J}_{1}(\cdot|s_g(t+1),s_g(t), a_g(t), \cdot)$, making $\mathcal{J}_{1}(\cdot|s_g(t+1),s_g(t), a_g(t), \cdot)$ column-stochastic. This yields the claim.\qedhere \\
\end{proof}

\begin{lemma} \emph{\label{lemma: tvd linear bound}The total variation distance between the expected empirical distribution of $s_\Delta({t+1})$ and $s_{\Delta'}({t+1})$ is linearly bounded by the total variation distance of the empirical distributions of $s_\Delta({t})$ and $s_{\Delta'}({t})$, for $\Delta,\Delta'\subseteq[n]$:
\[\mathrm{TV}\left(\E_{\substack{s_i({t+1})\sim P_l(\cdot| s_i(t),s_g(t)), \\ \forall i\in\Delta}}F_{s_\Delta({t+1})}, \E_{\substack{s_i({t+1})\sim P_l(\cdot| s_i(t),s_g(t)),\\ \forall i\in\Delta'}}F_{s_{\Delta'}({t+1})}\right) \leq \mathrm{TV}\left(F_{s_\Delta(t)}, F_{s_{\Delta'}(t)}\right)\]    }
\end{lemma}
\begin{proof}\label{lemma: generalized linear transition}
We expand the total variation distance measure in $\ell_1$-norm and utilize the result from \cref{lemma: linear_transition} that $\E_{\substack{s_i({t+1})\sim P_l(\cdot| s_i(t),s_g(t)) \\ \forall i\in\Delta}}F_{s_\Delta({t+1})} = P_l(\cdot|s_g(t)) F_{s_\Delta(t)}$ as follows:
    \begin{align*}
\mathrm{TV}&\left(\E_{\substack{s_i({t+1})\sim P_l(\cdot| s_i(t),s_g(t)) \\ \forall i\in\Delta}}F_{s_\Delta({t+1})}, \E_{\substack{s_i({t+1})\sim P_l(\cdot| s_i(t),s_g(t)) \\ \forall i\in\Delta'}}F_{s_{\Delta'}(t+1)}\right) \\
&\quad\quad\quad= \frac{1}{2}\left\|\E_{\substack{s_i({t+1})\sim P_l(\cdot| s_i(t),s_g(t)) \\ \forall i\in\Delta}}F_{s_\Delta({t+1})} - \E_{\substack{s_i({t+1})\sim P_l(\cdot| s_i(t),s_g(t)) \\ \forall i\in\Delta'}}F_{s_{\Delta'}{(t+1)}}\right\|_1 \\
&\quad\quad\quad= \frac{1}{2}\left\|P_l(\cdot|\cdot,s_g(t))F_{s_\Delta({t})} - P_l(\cdot|\cdot,s_g(t)) F_{s_{\Delta'}(t)}\right\|_1 \\
&\quad\quad\quad\leq \|P_l(\cdot|\cdot,s_g(t))\|_1 \cdot \frac{1}{2}|F_{s_\Delta(t)} - F_{s_{\Delta'}(t)}|_1 \\
&\quad\quad\quad= \|P_l(\cdot|\cdot,s_g(t))\|_1 \cdot \mathrm{TV}(F_{s_\Delta(t)}, F_{s_{\Delta'}(t)})
    \end{align*}
In the last line, we recover the total variation distance from the $\ell_1$ norm. Finally, by the column stochasticity of $P_l(\cdot|\cdot,s_g)$, we have that $\|P_l(\cdot|\cdot,s_g)\|_1 \leq 1$, which then implies \[\mathrm{TV}\left(\E_{\substack{s_i({t+1})\sim P_l(\cdot| s_i(t),s_g(t)) \\ \forall i\in\Delta}}F_{s_\Delta({t+1})}, \E_{\substack{s_i({t+1})\sim P_l(\cdot| s_i(t),s_g(t)) \\ \forall i\in\Delta'}}F_{s_{\Delta'}({t+1})}\right) \leq \mathrm{TV}(F_{s_\Delta(t)}, F_{s_{\Delta'}(t)})\]    
This proves the lemma.\qedhere
\end{proof}
\begin{remark}
\emph{    \cref{lemma: tvd linear bound} can be viewed as an irreducibility and aperiodicity result on the finite-state Markov chain whose state space is given by $\mathcal{S}=\mathcal{S}_g\times\mathcal{S}_l^n$. Let $\{s_t\}_{t\in\mathbb{N}}$ denote the sequence of states visited by this Markov chain where the transitions are induced by the transition functions $P_g, P_l$. Through this, \cref{lemma: tvd linear bound} describes an ergodic behavior of the Markov chain.\\}
\end{remark}

\begin{lemma}\emph{ The absolute difference between the expected maximums between $\hat{Q}_k$ and $\hat{Q}_{k'}$ is atmost the maximum of the absolute difference between $\hat{Q}_k$ and $\hat{Q}_{k'}$, where the expectations are taken over any joint distributions of states $\mathcal{J}$, and the maximums are taken over the actions.}
\label{lemma: expectation_expectation_max_swap}\begin{align*}|\E_{(s_g',s_{\Delta\cup\Delta'}')\sim\mathcal{J}_{|\Delta\cup\Delta'|}(\cdot,\cdot|s_g,a_g,s_{\Delta\cup\Delta'})}[\max_{a_g'\in\mathcal{A}_g} \hat{Q}_k^T(s_g',F_{s_\Delta'},a_g') &-\max_{a_g'\in\mathcal{A}_g} \hat{Q}_{k'}^T(s_g',F_{s_{\Delta'}'}, a_g')]| \\
\leq \max_{a_g'\in\mathcal{A}_g}|\E_{(s_g',s_{\Delta\cup\Delta'}')\sim\mathcal{J}_{|\Delta\cup\Delta'|}(\cdot,\cdot|s_g,a_g,s_{\Delta\cup\Delta'})}[\hat{Q}_k^T(s_g',F_{s_\Delta'},a_g') &- \hat{Q}_{k'}^T(s_g',F_{s_{\Delta'}'},a_g')]|\end{align*}
\end{lemma}
\begin{proof}
\[a_g^* := \arg\max_{a_g'\in\mathcal{A}_g}\hat{Q}_k^T(s_g', F_{s'_\Delta}, a_g'),\quad\tilde{a}_g^* := \arg\max_{a_g'\in\mathcal{A}_g}\hat{Q}_{k'}^T(s_g', F_{s'_{\Delta'}}, a_g')\]
For the remainder of this proof, we adopt the shorthand $\E_{s_g',s_{\Delta\cup\Delta'}'}$ to refer to $\E_{(s_g',s_{\Delta\cup\Delta'}')\sim \mathcal{J}_{|\Delta\cup\Delta'|}(\cdot,\cdot|s_g, a_g, s_{\Delta\cup\Delta'})}$.\\

Then, if $\E_{s_g',s_{\Delta\cup\Delta'}'} \max_{a_g'\in\mathcal{A}_g}\hat{Q}_k^T(s_g', F_{s'_\Delta}, a_g') - \E_{s_g',s_{\Delta\cup\Delta'}'}\max_{a_g'\in\mathcal{A}_g}\hat{Q}_{k'}^T(s_g', F_{s'_{\Delta'}}, a_g')>0$, we have:
\begin{align*}
&|\E_{s_g',s_{\Delta\cup\Delta'}'} \max_{a_g'\in\mathcal{A}_g}\hat{Q}_k^T(s_g', F_{s'_\Delta}, a_g') - \E_{s_g',s_{\Delta\cup\Delta'}'}\max_{a_g'\in\mathcal{A}_g}\hat{Q}_{k'}^T(s_g', F_{s'_{\Delta'}}, a_g')| \\
&\quad\quad\quad\quad= \E_{s_g',s_{\Delta\cup\Delta'}'} \hat{Q}_k^T(s_g', F_{s'_\Delta}, a_g^*) -\E_{s_g',s_{\Delta\cup\Delta'}'} \hat{Q}_{k'}^T(s_g', F_{s'_{\Delta'}}, \tilde{a}_g^*) \\
&\quad\quad\quad\quad\leq \E_{s_g',s_{\Delta\cup\Delta'}'}\hat{Q}^T_k(s_g', F_{s'_\Delta}, a_g^*) - \E_{s_g',s_{\Delta\cup\Delta'}'} \hat{Q}_{k'}^T(s_g', F_{s'_{\Delta'}}, a_g^*) \\
&\quad\quad\quad\quad\leq \max_{a_g'\in\mathcal{A}_g}|\E_{s_g',s_{\Delta\cup\Delta'}'}\hat{Q}_k^T(s_g', F_{s_\Delta'}, a_g') - \E_{s_g',s_{\Delta\cup\Delta'}'}\hat{Q}_{k'}^T(s_g', F_{s'_{\Delta'}}, a_g')|
\end{align*}
Similarly, if $\E_{s_g',s_{\Delta\cup\Delta'}'} \max_{a_g'\in\mathcal{A}_g}\hat{Q}_k^T(s_g', F_{s'_\Delta}, a_g') - \E_{s_g',s_{\Delta\cup\Delta'}'}\max_{a_g'\in\mathcal{A}_g}\hat{Q}_{k'}^T(s_g', F_{s'_{\Delta'}}, a_g')<0$, an analogous argument by replacing $a_g^*$ with $\tilde{a}_g^*$ yields an identical bound. \qedhere \\
\end{proof}

\begin{lemma} \label{lemma: linear_transition}For all $t\in\N$ and $\Delta\subseteq[n]$,
    \[    \E_{\substack{s_i({t+1})\sim P_l(\cdot|s_i(t), s_g(t))\\ \forall i\in\Delta}} [F_{s_\Delta({t+1})}] = P_l(\cdot|\cdot,s_g(t))F_{s_\Delta(t)}\]
\end{lemma}
\begin{proof} 
For all  $x\in\mathcal{S}_l$:
\begin{align*}
\E_{\substack{s_i({t+1})\sim P_l(\cdot|s_i(t), s_g(t))\\ \forall i\in\Delta}} [F_{s_\Delta({t+1})}(x)] &:= \frac{1}{|\Delta|}\sum_{i\in\Delta}\E_{s_i({t+1})\sim P_l(s_i(t),s_g(t))}[\mathbbm{1}(s_i({t+1})=x)] \\
&= \frac{1}{|\Delta|}\sum_{i\in\Delta}\Pr[s_i({t+1}) = x|s_i({t+1})\sim P_l(\cdot|s_i(t),s_g(t))]  \\
&= \frac{1}{|\Delta|}\sum_{i\in\Delta}P_l(x|s_i(t),s_g(t)) 
\end{align*}
In the first line, we are writing out the definition of $F_{s_\Delta({t+1})}(x)$ and using the conditional independence in the evolutions of $\Delta\setminus i$ and $i$. In the second line, we use the fact that for any random variable $X\in\mathcal{X}$, $\E_{X}\mathbbm{1}[X=x] = \Pr[X=x]$. In line 3, we observe that the above probability can be written as an entry of the local transition matrix $P_l$. Then, aggregating across every entry $x\in\mathcal{S}_l$, we have that:
\begin{align*}
\E_{\substack{s_i({t+1})\sim P_l(\cdot|s_i(t), s_g(t))\\ \forall i\in\Delta}}[F_{s_\Delta({t+1})}] &= \frac{1}{|\Delta|}\sum_{i\in\Delta}P_l(\cdot|s_i(t),s_g(t)) \\
 &= \frac{1}{|\Delta|}\sum_{i\in\Delta}P_l(\cdot|\cdot,s_g(t))\vec{\mathbbm{1}}_{s_i(t)} =: P_l(\cdot|\cdot,s_g(t)) F_{s_\Delta(t)}
\end{align*}
Here, $\vec{\mathbbm{1}}_{s_i(t)}\in  \{0,1\}^{|\mathcal{S}_l|}$ such that $\vec{\mathbbm{1}}_{s_i(t)}$ is $1$ at the index corresponding to $s_i(t)$, and is $0$ everywhere else. The last equality follows since $P_l(\cdot|\cdot,s_g(t))$ is a column-stochastic matrix which yields that $P_l(\cdot|\cdot,s_g(t))\vec{\mathbbm{1}}_{s_i(t)}=P_l(\cdot|s_i(t),s_g(t)) $, thus proving the lemma.\qedhere\\
\end{proof}

\begin{lemma}\label{lemma: combining transition probabilities}
For any joint transition probability function on $s_g, s_\Delta$, where $|\Delta|=k$, given by $\mathcal{J}_k:\mathcal{S}_g\times\mathcal{S}_l^{|\Delta|}\times\mathcal{S}_g\times\mathcal{A}_g\times\mathcal{S}_l^{|\Delta|}\to [0,1]$, we have:
\begin{align*}
&\E_{(s_g',s_\Delta')\sim\mathcal{J}_k(\cdot,\cdot|s_g,a_g,s_\Delta)}\left[\E_{(s_g'',s_\Delta'')\sim\mathcal{J}_k(\cdot,\cdot|s_g',a_g,s_\Delta')} \max_{a_g''\in\mathcal{A}_g}\hat{Q}_k^T(s_g'',F_{s_\Delta''},a_g'')\right] \\ 
&\quad\quad= \E_{(s_g'',s_\Delta'')\sim\mathcal{J}_k^2(\cdot,\cdot|s_g,a_g,s_\Delta)} \max_{a_g''\in\mathcal{A}_g}\hat{Q}_{k}^T(s_g'',F_{s_\Delta''},a_g'')
\end{align*}
\end{lemma}
\begin{proof}
We start by expanding the expectations:
\begin{align*}
&\E_{(s_g',s_\Delta')\sim \mathcal{J}_k(\cdot,\cdot|s_g,a_g,s_\Delta)}\left[\E_{(s_g'',s_\Delta'')\sim \mathcal{J}_k(\cdot,\cdot|s_g',a_g,s_\Delta')}\max_{a_g'\in\mathcal{A}_g}\hat{Q}_k^T(s_g'',F_{s_\Delta''},a_g')\right] \\
&= \!\!\sum_{(s_g',s_\Delta')\in\mathcal{S}_g\times\mathcal{S}_l^{|\Delta|}}\!\sum_{(s_g'',s_\Delta'')\in\mathcal{S}_g\times\mathcal{S}_l^{|\Delta|}}\!\!\!\!\!\mathcal{J}_k[s_g',s_\Delta', s_g,a_g,s_\Delta]\mathcal{J}_k[s_g'',s_\Delta'', s_g',a_g,s_\Delta']\max_{a_g'\in\mathcal{A}_g}\hat{Q}_k^T(s_g'',F_{s_\Delta''},a_g') \\
&= \sum_{(s_g'',s_\Delta'')\in\mathcal{S}_g\times\mathcal{S}_l^{|\Delta|}}\mathcal{J}_k^2[s_g'',s_\Delta'', s_g, a_g, s_\Delta] \max_{a_g'\in\mathcal{A}_g}\hat{Q}^T_k(s_g'',F_{s_\Delta''},a_g') \\
&= \E_{(s_g'',s_\Delta'')\sim\mathcal{J}_k^2(\cdot,\cdot|s_g,a_g,s_\Delta)}\max_{a_g'\in\mathcal{A}_g}\hat{Q}_k^T(s_g'',F_{s_\Delta''},a_g')
\end{align*}
\noindent The right-stochasticity of $\mathcal{J}_k$ implies the right-stochasticity of $\mathcal{J}_k^2$. 
Further, observe that $\mathcal{J}_k[s_g',s_\Delta',s_g,a_g,s_\Delta]\mathcal{J}_k[s_g'',s_\Delta'',s_g',a_g,s_\Delta']$ denotes the probability of the transitions $(s_g,s_\Delta)\to (s_g',s_\Delta')\to (s_g'',s_\Delta'')$ with actions $a_g$ at each step, where the joint state evolution is governed by $\mathcal{J}_k$. Thus, $\sum_{(s_g',s_\Delta')\in\mathcal{S}_g\times\mathcal{S}_l^{|\Delta|}}\mathcal{J}_k[s_g',s_\Delta',s_g,a_g,s_\Delta]\mathcal{J}_k[s_g'',s_\Delta'',s_g', a_g, s_g']$ is the stochastic probability function corresponding to the two-step evolution of the joint states from $(s_g,s_\Delta)$ to $(s_g'',s_\Delta'')$ under the action $a_g$, which is equivalent to $\mathcal{J}_k^2[s_g'',s_\Delta'',s_g,a_g,s_\Delta]$. In the third equality, we recover the definition of the expectation, where the joint probabilities are taken over $\mathcal{J}_k^2$. \qedhere
\end{proof}

The following lemma bounds the average difference between $\hat{Q}_k^T$ (across every choice of $\Delta\in\binom{[n]}{k}$) and $Q^*$ and shows that the difference decays to $0$ as $T\to\infty$.
\begin{lemma}\label{lemma:comparing_expectation}
For all $s\in \mathcal{S}_g\times \mathcal{S}_{[n]}$, and for all $a_g\in \mathcal{A}_g$, we have:
\[Q^*(s,a_g) - \frac{1}{\binom{n}{k}}\sum_{\Delta \in \binom{[n]}{k}}\hat{Q}_k^T(s_g,F_{s_\Delta},a_g) \leq \gamma^T\frac{\tilde{r}}{1-\gamma}\]
\end{lemma}
\begin{proof}

We bound the differences between $\hat{Q}_k^T$ at each Bellman iteration of our approximation to $Q^*$.
\begin{align*}
    &Q^*(s,a_g) - \frac{1}{\binom{n}{k}}\sum_{\Delta \in \binom{[n]}{k}}\hat{Q}_k^T(s_g,F_{s_{\Delta}},a_g) 
    \\
    &= \mathcal{T}Q^*(s,a_g) - \frac{1}{\binom{n}{k}}\sum_{\Delta \in \binom{[n]}{k}}\hat{\mathcal{T}}_k\hat{Q}_k^{T-1}(s_g,F_{s_{\Delta}},a_g) \\ 
    &=  r_{[n]}(s_g,s_{[n]},a_g) + \gamma 
    \E_{\substack{s_g'\sim P_g(\cdot|s_g,a_g),\\ s_i'\sim P_l(\cdot|s_i,s_g),\forall i\in[n])}}
    \max_{a_g'\in \mathcal{A}_g} Q^*(s',a_g') \\
    &\quad\quad\quad\quad - \frac{1}{\binom{n}{k}}\sum_{\Delta\in \binom{[n]}{k}} [r_{[\Delta]}(s_g,s_\Delta,a_g) + \gamma \E_{\substack{s_g'\sim{P}_g(\cdot|s_g,a_g)\\s_i'\sim {P}_l(\cdot|s_i,s_g), \forall i\in\Delta}}\max_{a_g'\in \mathcal{A}_g} Q_k^T(s_g', F_{s_\Delta'},a_g')]
\end{align*}

\noindent Next, observe that $r_{[n]}(s_g,s_{[n]}, a_g) = \frac{1}{\binom{n}{k}}\sum_{\Delta\in \binom{[n]}{k}} r_{[\Delta]}(s_g,s_\Delta,a_g)$. To prove this, we write:
\begin{align*}
    \frac{1}{\binom{n}{k}}\sum_{\Delta\in \binom{[n]}{k}} r_{[\Delta]}(s_g,s_\Delta,a_g) &= \frac{1}{\binom{n}{k}}\sum_{\Delta\in \binom{[n]}{k}} (r_g(s_g,a_g)+ \frac{1}{k} \sum_{i\in\Delta}r_l(s_i, s_g)) \\
    &= r_g(s_g,a_g) + \frac{\binom{n-1}{k-1}}{k\binom{n}{k}}\sum_{i\in[n]} r_l(s_i,s_g) \\
    &= r_g(s_g,a_g) + \frac{1}{n}\sum_{i\in[n]}r_l(s_i, s_g) := r_{[n]}(s_g, s_{[n]}, a_g)
\end{align*}
In the second equality, we reparameterized the sum to count the number of times each $r_l(s_i, s_g)$ was added for each $i\in\Delta$, and in the last equality, we expanded and simplified the binomial coefficients. Therefore:
\begin{align*}
&\sup_{(s,a_g)\in\mathcal{S}\times\mathcal{A}_g}[Q^*(s,a_g) - \frac{1}{\binom{n}{k}}\sum_{\Delta \in \binom{[n]}{k}}\hat{Q}_k^T (s_g,F_{s_{[n]}},a_g)] \\
&= \sup_{(s,a_g)\in\mathcal{S}\times\mathcal{A}_g}[\mathcal{T}Q^*(s,a_g) - \frac{1}{\binom{n}{k}}\sum_{\Delta \in \binom{[n]}{k}}\hat{\mathcal{T}}_k\hat{Q}_k^{T-1}(s_g,F_{s_{[n]}},a_g)] \\
    &= \gamma\sup_{(s,a_g)\in\mathcal{S}\times\mathcal{A}_g}\!\![\mathbb{E}_{\substack{s_g'\sim {P}(\cdot|s_g,a_g) \\ s_i'\sim P_l(\cdot|s_i,s_g)\\ \forall i\in[n]}}\!\!\max_{a_g'\in \mathcal{A}_g}Q^*(s',a_g')\!-\!\!\frac{1}{\binom{n}{k}}\!\!\!\sum_{\Delta\in \binom{[n]}{k}}\!\!\!\E_{\substack{s_g'\sim {P}_g(\cdot|s_g,a_g) \\ s_i'\sim {P}_l(\cdot|s_i,s_g) \\ \forall i\in\Delta}}\max_{a_g'\in\mathcal{A}_g}\hat{Q}_k^{T-1}(s_g',F_{s_\Delta'},a_g')] \\
    &= \gamma\sup_{(s,a_g)\in\mathcal{S}\times\mathcal{A}_g} \E_{\substack{s_g'\sim {P}_g(\cdot|s_g,a_g), \\ s_i'\sim P_l(\cdot|s_i,s_g), \forall i\in [n]}} [\max_{a_g'\in\mathcal{A}_g}Q^*(s',a_g') - \frac{1}{\binom{n}{k}}\sum_{\Delta\in\binom{[n]}{k}}\max_{a_g'\in\mathcal{A}_g}\hat{Q}_k^{T-1}(s_g',F_{s_\Delta'},a_g')]\\
    &\leq \gamma\sup_{(s,a_g)\in\mathcal{S}\times\mathcal{A}_g} \E_{\substack{s_g'\sim {P}_g(\cdot|s_g,a_g), \\ s_i'\sim P_l(\cdot|s_i,s_g), \forall i\in [n]}}\max_{a_g'\in\mathcal{A}_g}[Q^*(s',a_g') - \frac{1}{\binom{n}{k}}\sum_{\Delta\in\binom{[n]}{k}}\hat{Q}_k^{T-1}(s_g',F_{s_\Delta'},a_g')] \\
    &\leq \gamma \sup_{(s',a'_g)\in  \mathcal{S}\times \mathcal{A}_g} [Q^*(s',a_g') - \frac{1}{\binom{n}{k}}\sum_{\Delta\in\binom{[n]}{k}}\hat{Q}_k^{T-1}(s'_g,F_{s'_\Delta},a'_g)]
\end{align*}
We justify the first inequality by noting the general property that for positive vectors $v, v'$ for which $v\succeq v'$ which follows from the triangle inequality:
\begin{align*}\|v - \frac{1}{\binom{n}{k}}\sum_{\Delta\in\binom{[n]}{k}}v'\|_\infty &\geq |\|v\|_\infty - \|\frac{1}{\binom{n}{k}}\sum_{\Delta\in\binom{[n]}{k}}v'\|_\infty| \\
&= \|v\|_\infty - \|\frac{1}{\binom{n}{k}}\sum_{\Delta\in\binom{[n]}{k}}v'\|_\infty \\
&\geq \|v\|_\infty - \frac{1}{\binom{n}{k}}\sum_{\Delta\in\binom{[n]}{k}}\|v'\|_\infty\end{align*}
Therefore:
\begin{align*}Q^*(s,a_g) &- \frac{1}{\binom{n}{k}}\sum_{\Delta\in \binom{[n]}{k}}\hat{Q}_k^T(s_g,F_{s_\Delta},a_g) \\
&\leq \gamma^{T} \sup_{(s', a_g) \in \mathcal{S}\times \mathcal{A}_g}[Q^*(s',a'_g) - \frac{1}{\binom{n}{k}}\sum_{\Delta\in\binom{[n]}{k}}\hat{Q}_k^0(s'_g,F_{s'_\Delta},a'_g)] \\
&= \frac{\gamma^T\tilde{r}}{1-\gamma}
\end{align*}
\noindent The first inequality follows from the $\gamma$-contraction property of the update procedure, and the ensuing equality follows from our bound on the maximum possible value of $Q$ from \cref{lemma: Q-bound} and noting that $\hat{Q}_k^0 := 0$. Therefore, as $T\to\infty$, $Q^*(s,a_g) - \frac{1}{\binom{n}{k}}\sum_{\Delta\in\binom{[n]}{k}}\hat{Q}^T(s_g,F_{s_\Delta},a_g) \to 0$, which proves the lemma.\qedhere \\
\end{proof}

\section{Bounding Total Variation Distance}

As $|\Delta| \to n$, the total variation (TV) distance between the empirical distribution of $s_{[n]}$ and $s_\Delta$ goes to $0$. We formalize this notion and prove this statement by obtaining tight bounds on the difference and showing that this error decays quickly. 

\begin{remark} \emph{First, observe that if $\Delta$ is an independent random variable uniformly supported on $\binom{[n]}{k}$, then $s_{\Delta}$ is also an independent random variable uniformly supported on the global state $\binom{s_{[n]}}{k}$. To see this, let $\psi_1:[n]\to \mathcal{S}_l$ where $\psi(i) = s_i$. This naturally extends to $\psi_k: [n]^k\to \mathcal{S}_l^k$ given by $\psi_k(i_1,\dots,i_k) = (s_{i_1},\dots,s_{i_k})$, for all $k\in[n]$. Then, the independence of $\Delta$ implies the independence of the generated $\sigma$-algebra. Further, $\psi_k$ (which is a Lebesgue measurable function of a $\sigma$-algebra) is a sub-algebra, implying that $s_\Delta$ must also be an independent random variable.} \end{remark}
For reference, we present the multidimensional Dvoretzky-Kiefer-Wolfowitz (DKW) inequality \cite{10.1214/aoms/1177728174,10.1214/aop/1176990746,NAAMAN2021109088} which bounds the difference between an empirical distribution function for $s_\Delta$ and $s_{[n]}$ when each element of $\Delta$ for $|\Delta|=k$ is sampled uniformly randomly from $[n]$ \emph{with} replacement.

\begin{theorem}[Dvoretzky-Kiefer-Wolfowitz (DFW) inequality \cite{10.1214/aoms/1177728174}] By the multi-dimensional version of the DKW inequality \cite{NAAMAN2021109088}, assume that $\mathcal{S}_l \subset \mathbb{R}^d$. Then, for any $\epsilon>0$, the following statement holds for when $\Delta\subseteq[n]$ is sampled uniformly \emph{with} replacement.
\[\Pr\left[\sup_{x\in\mathcal{S}_l}\left|\frac{1}{|\Delta|}\sum_{i\in\Delta}\mathbbm{1}\{s_i = x\} - \frac{1}{n}\sum_{i=1}^n \mathbbm{1}\{s_i = x\}\right| < \epsilon\right] \geq 1 - d(n+1)e^{-2|\Delta|\epsilon^2}\cdot\]
\end{theorem}

We give an analogous bound for the case when $\Delta$ is sampled uniformly from $[n]$ without replacement. However, our bound does \emph{not} have a dependency on $d$, the dimension of $\mathcal{S}_l$ which allows us to consider non-numerical state-spaces.

Before giving the proof, we add a remark on this problem. Intuitively, when samples are chosen without replacement from a finite population, the marginal distribution, when conditioned on the random variable chosen, takes the running empirical distribution closer to the true distribution with high probability. However, we need a uniform probabilistic bound on the error that adapts to \emph{worst-case marginal distributions} and decays with $k$.\\

Recall the landmark results of Hoeffding and Serfling in \cite{409cf137-dbb5-3eb1-8cfe-0743c3dc925f} and \cite{faaf16cc-812f-3d2c-9158-610f115420d7} which we restate below.

\begin{lemma}[Lemma 4, Hoeffding]\label{lemma: hoeffding_lemma_4} Given a finite population, note that for any convex and continuous function $f:\mathbb{R}\to\mathbb{R}$, if $X = \{x_1,\dots,x_k\}$ denotes a sample with replacement and $Y = \{y_1,\dots,y_k\}$ denotes a sample without replacement, then:
\[\mathbb{E} f \left(\sum_{i\in X} i\right) \leq \mathbb{E} f\left(\sum_{i\in Y}i\right)\]
\end{lemma}

\begin{lemma}[Corollary 1.1, Serfling]\label{lemma: serfling corollary 1.1}
    Suppose the finite subset $\mathcal{X}\subset\mathbb{R}$ such that $|\mathcal{X}|=n$ is bounded between $[a,b]$. Then, let $X=(x_1,\dots,x_k)$ be a random sample of $\mathcal{X}$ of size $k$ chosen uniformly and without replacement. Denote $\mu:=\frac{1}{n}\sum_{i=1}^n x_i$. Then:
    \[\Pr\left[\left|\frac{1}{k}\sum_{i=1}^kx_i - \mu\right|>\epsilon\right] < 2 e^{-\frac{2k\epsilon^2}{(b-a)^2(1-\frac{k-1}{n})}}\]
\end{lemma}

\noindent We now present a sampling \emph{without} replacement analog of the DKW inequality.\\

\begin{theorem}[Sampling without replacement analogue of the DKW inequality]\label{theorem: sampling without replacement analog of DKW}
Consider a finite population $\mathcal{X}=(x_1,\dots,x_n)\in \mathcal{S}_l^n$. Let $\Delta\subseteq[n]$ be a random sample of size $k$ chosen uniformly and without replacement. 
 
 Then, for all $x\in\mathcal{S}_l$:
\[\Pr\left[\sup_{x\in\mathcal{S}_l}\left|\frac{1}{|\Delta|}\sum_{i\in\Delta}\mathbbm{1}{\{x_i = x\}} - \frac{1}{n}\sum_{i\in[n]}\mathbbm{1}{\{x_i = x\}}\right|<\epsilon\right] \geq 1 - 2|\mathcal{S}_l|e^{-\frac{2|\Delta|n\epsilon^2}{n-|\Delta|+1}}\]
\end{theorem}
\begin{proof}
\noindent For each $x\in\mathcal{S}_l$, define the ``$x$-surrogate population'' of indicator variables as \begin{equation}\bar{\mathcal{X}}_x = (\mathbbm{1}_{\{x_1 = x\}}, \dots, \mathbbm{1}_{\{x_n = x\}}) \in \{0,1\}^n \end{equation}
Since the maximal difference between each element in this surrogate population is $1$, we set $b-a=1$ in \cref{lemma: serfling corollary 1.1} when applied to $\bar{\mathcal{X}}_x$ to get:
\[\Pr\left[\left|\frac{1}{|\Delta|}\sum_{i\in\Delta}\mathbbm{1}{\{x_i = x\}} - \frac{1}{n}\sum_{i\in[n]}\mathbbm{1}{\{x_i = x\}}\right|<\epsilon\right] \geq 1 - 2e^{-\frac{2|\Delta|n\epsilon^2}{n-|\Delta|+1}}\]
In the above equation, the probability is over $\Delta\subseteq\binom{[n]}{k}$ and it holds for each $x\in\mathcal{S}_l$. Therefore, the randomness is only over $\Delta$. Then, by a union bounding argument, we have:
\begin{align*}
\Pr\bigg[\sup_{x\in\mathcal{S}_l}\bigg|\frac{1}{|\Delta|}\sum_{i\in\Delta}\mathbbm{1}{\{x_i = x\}} \!&-\! \frac{1}{n}\sum_{i\in[n]}\mathbbm{1}{\{x_i = x\}}\bigg|<\epsilon\bigg]\\
&= \Pr\left[\bigcap_{x\in\mathcal{S}_l}\left\{\left|\frac{1}{|\Delta|}\sum_{i\in\Delta}\mathbbm{1}{\{x_i = x\}} \!-\! \frac{1}{n}\sum_{i\in[n]}\mathbbm{1}{\{x_i = x\}}\right|<\epsilon\right\}\right] \\
&= 1 \!-\! \sum_{x\in\mathcal{S}_l}\Pr\left[\left|\frac{1}{|\Delta|}\sum_{i\in\Delta}\mathbbm{1}{\{x_i = x\}} \!-\! \frac{1}{n}\sum_{i\in[n]}\mathbbm{1}{\{x_i = x\}}\right|\geq\epsilon\right]\\
&\geq 1 \!-\! 2|\mathcal{S}_l|e^{-\frac{2|\Delta|n\epsilon^2}{n-|\Delta|+1}}\end{align*}
This proves the claim.\qedhere \\
\end{proof}

Then, combining the Lipschitz continuity bound from \cref{thm:lip} and the total variation distance bound from \cref{thm:tvd} yields \cref{theorem: Q-lipschitz of Fsdelta and Fsn}. \\

\begin{theorem}\label{theorem: Q-lipschitz of Fsdelta and Fsn}
    For all $s_g\in\mathcal{S}_g, s_1,\dots,s_n\in\mathcal{S}_l^n, a_g\in\mathcal{A}_g$, we have that with probability atleast $1 - \delta$:
\[|\hat{Q}_k^T(s_g, F_{s_\Delta}, a_g) - \hat{Q}_n^T(s_g, F_{s_{[n]}}, a_g)| \leq \frac{2\|r_l(\cdot,\cdot)\|_\infty}{1-\gamma}\sqrt{\frac{n-|\Delta|+1}{8n|\Delta|}\ln ( 2|\mathcal{S}_l|/\delta)}\]

\begin{proof}
By the definition of total variation distance, observe that 
\begin{equation}\label{equation: tv equivalence epsilon}\mathrm{TV}(F_{s_\Delta}, F_{s_{[n]}})\leq \epsilon \iff \sup_{x\in\mathcal{S}_l} |F_{s_\Delta} - F_{s_{[n]}}| < 2\epsilon\end{equation}
Then, let $\mathcal{X}=\mathcal{S}_l$ be the finite population in \cref{theorem: sampling without replacement analog of DKW} and
    recall the Lipschitz-continuity of $\hat{Q}_k^T$ from \cref{lemma: Q lipschitz continuous wrt Fsdelta}:
\begin{align*}\left|\hat{Q}_k^T(s_g,F_{s_\Delta},a_g) - \hat{Q}_n^T(s_g, F_{s_{[n]}}, a_g)\right| &\leq \left(\sum_{t=0}^{T-1} 2\gamma^t\right)\|r_l(\cdot,\cdot)\|_\infty \cdot \mathrm{TV}(F_{s_\Delta}, F_{s_{[n]}}) \\
&\leq \frac{2}{1-\gamma}\|r_l(\cdot,\cdot)\|_\infty\cdot\epsilon
\end{align*}
By setting the error parameter in \cref{theorem: sampling without replacement analog of DKW} to $2\epsilon$, we find that \cref{equation: tv equivalence epsilon} occurs with probability at least $1-2|\mathcal{S}_l|e^{-2|\Delta|n\epsilon^2/(n-|\Delta|+1)}$.
\[\Pr\left[\left|\hat{Q}_k^T(s_g,F_{s_\Delta},a_g) - \hat{Q}_n^T(s_g, F_{s_{[n]}}, a_g)\right| \leq \frac{2\epsilon}{1-\gamma}\|r_l(\cdot,\cdot)\|_\infty\right]\geq 1 - 2|\mathcal{S}_l|e^{-\frac{8n|\Delta|\epsilon^2}{n-|\Delta|+1}}\]
Finally, we parameterize the probability to $1-\delta$ to solve for $\epsilon$, which yields \[\epsilon = \sqrt{\frac{n-|\Delta|+1}{8n|\Delta|}\ln(2|\mathcal{S}_l|/\delta)}.\]
This proves the theorem.\qedhere \\
\end{proof}
\end{theorem}

The following lemma is not used in the main result; however, we include it to demonstrate why popular TV-distance bounding methods using the Kullback-Liebler (KL) divergence and the Bretagnolle-Huber inequality \cite{10.5555/1522486} only yield results with a suboptimal subtractive decay of $\sqrt{|\Delta|/n}$. In comparison, \cref{thm:tvd} achieves a stronger multiplicative decay of $1/\sqrt{|\Delta|}$.\\

\begin{lemma}\label{lemma: tv_distance_bretagnolle_huber}
\[\mathrm{TV}(F_{s_\Delta}, F_{s_{[n]}}) \leq \sqrt{1- |\Delta|/n}\]
\end{lemma}
\begin{proof}
    By the symmetry of the total variation distance, we have $\mathrm{TV}(F_{s_{[n]}}, F_{s_{\Delta}}) = \mathrm{TV}(F_{s_{\Delta}}, F_{s_{[n]}})$. 
    
    From the Bretagnolle-Huber inequality \cite{10.5555/1522486} we have that
    $\mathrm{TV}(f, g) = \sqrt{1 - e^{-D_{\mathrm{KL}}(f\|g)}}$. Here, $D_{\mathrm{KL}}(f\|g)$ is the Kullback-Leibler (KL) divergence metric between probability distributions $f$ and $g$ over the sample space, which we denote by $\mathcal{X}$ and is given by \begin{equation}\label{def:kl_div}D_{\mathrm{KL}}(f\|g) := \sum_{x\in \mathcal{X}}f(x)\ln\frac{f(x)}{g(x)}\end{equation}Thus, from \cref{def:kl_div}:
\begin{align*}
    D_{\mathrm{KL}}(F_{s_\Delta}\|F_{s_{[n]}}) &= \sum_{x\in\mathcal{S}_l}\left(\frac{1}{|\Delta|} \sum_{i\in\Delta} \mathbbm{1}\{s_i = x\}\right)\ln \frac{n\sum_{i\in\Delta}\mathbbm{1}\{s_i=x\}}{|\Delta|\sum_{i\in[n]}\mathbbm{1}\{s_i=x\}} \\
    &= \frac{1}{|\Delta|}\sum_{x\in\mathcal{S}_l}\left(\sum_{i\in\Delta} \mathbbm{1}\{s_i = x\}\right)\ln \frac{n}{|\Delta|} \\
    &\quad\quad\quad\quad\quad\quad+ \frac{1}{|\Delta|}\sum_{x\in\mathcal{S}_l}\left(\sum_{i\in\Delta} \mathbbm{1}\{s_i = x\}\right) \ln \frac{\sum_{i\in\Delta}\mathbbm{1}\{s_i=x\}}{\sum_{i\in[n]}\mathbbm{1}\{s_i=x\}} \\
    &= \ln \frac{n}{|\Delta|} + \frac{1}{|\Delta|}\sum_{x\in\mathcal{S}_l}\left(\sum_{i\in\Delta} \mathbbm{1}\{s_i = x\}\right) \ln \frac{\sum_{i\in\Delta}\mathbbm{1}\{s_i=x\}}{\sum_{i\in[n]}\mathbbm{1}\{s_i=x\}} \\
    &\leq \ln (n/|\Delta|)
\end{align*}
In the third line, we note that $\sum_{x\in\mathcal{S}_l}\sum_{i\in\Delta}\mathbbm{1}\{s_i=x\} = |\Delta|$ since each local agent contained in $\Delta$ must have some state contained in $\mathcal{S}_l$. In the last line, we note that $\sum_{i\in\Delta}\mathbbm{1}\{s_i=x\}\leq \sum_{i\in[n]}\mathbbm{1}\{s_i = x\}$, For all $x\in \mathcal{S}_l$, and thus the summation of logarithmic terms in the third line is negative. Finally, using this bound in the Bretagnolle-Huber inequality yields the lemma.\qedhere \\
\end{proof}

\section{Using the Performance Difference Lemma to Bound the Optimality Gap}
Recall from \cref{defn:qkmest} that the fixed-point of the empirical adapted Bellman operator $\hat{\mathcal{T}}_{k,m}$ is $\hat{Q}^\est_{k,m}$. Further, recall from \cref{assumption:qest_qhat_error} that $\|\hat{Q}^*_k - \hat{Q}^\est_{k,m}\|_\infty \leq \epsilon_{k,m}$.\\

\begin{lemma}\label{lemma:union_bound_over_finite_time}
\emph{Fix $s\in \mathcal{S}:=\mathcal{S}_g\times\mathcal{S}_l^n$. Suppose we are given a $T$-length sequence of i.i.d. random variables $\Delta_1, \dots, \Delta_T$, distributed uniformly over the support $\binom{[n]}{k}$. Further, suppose we are given a fixed sequence $\delta_1,\dots,\delta_T \in (0,1)$. Then, for each action $a_g\in\mathcal{A}_g$ and for $i\in[T]$, define events $B_i^{a_g}$ such that:}
    \begin{align*}
        B_i^{a_g}\!\!:=\!\! \left\{\left|{Q}^*(s_g, s_{[n]}, a_g)\!-\! \hat{Q}_{k,m}^{\est}(s_g, F_{s_{\Delta_i}}, a_g)\right|\!>\! \sqrt{\frac{n-k+1}{8kn}\ln \frac{2|\mathcal{S}_l|}{\delta_i}}\cdot \frac{2}{1-\gamma}
    \|r_l(\cdot,\cdot)
    \|_\infty + \epsilon_{k,m}\right\}
    \end{align*}\emph{
    Next, for $i\in[M]$, we define ``bad-events'' $B_i$ such that
    $B_i = \bigcup_{a_g\in\mathcal{A}_g} B_i^{a_g}$. Next, denote $B = \cup_{i=1}^T B_i$.
    Then, the probability that  no ``bad event'' occurs is:
    \[\Pr\left[\bar{B}\right] \geq 1- |\mathcal{A}_g|\sum_{i=1}^T \delta_i\]}
\end{lemma}
\begin{proof}
\begin{align*}
    \left|{Q}^*(s_g, s_{[n]}, a_g) - \hat{Q}_{k,m}^{\est}(s_g, F_{s_\Delta}, a_g)\right| &\leq \left|{Q}^*(s_g, s_{[n]}, a_g) - \hat{Q}_k^*(s_g, F_{s_\Delta}, a_g)\right| \\
    &+ \left| \hat{Q}_k^*(s_g, F_{s_\Delta}, a_g) - \hat{Q}_{k,m}^{\est}(s_g, F_{s_\Delta}, a_g) \right| \\
    &\leq \left|{Q}^*(s_g, s_{[n]}, a_g) - \hat{Q}_k^*(s_g, F_{s_\Delta}, a_g)\right| + \epsilon_{k,m}
\end{align*}
The first inequality above follows from the triangle inequality, and the second inequality uses $|{Q}^*(s_g, s_{[n]}, a_g) - \hat{Q}_k^*(s_g, F_{s_\Delta}, a_g)|\leq \|{Q}^*(s_g, s_{[n]}, a_g) - \hat{Q}_k^*(s_g, F_{s_\Delta}, a_g)\|_\infty \leq \epsilon_{k,m}$, where $\epsilon_{k,m}$ is defined in \cref{assumption:qest_qhat_error}. Then, from \cref{theorem: Q-lipschitz of Fsdelta and Fsn}, we have that with probability at least $1-\delta_i$, 
\[\left|{Q}^*(s_g, s_{[n]}, a_g) - \hat{Q}_k^*(s_g, F_{s_\Delta}, a_g)\right| \leq \sqrt{\frac{n-k+1}{8nk} \ln\frac{2|\mathcal{S}_l|}{\delta_i}}\cdot \frac{2}{1-\gamma} \|r_l(\cdot,\cdot)\|_\infty\]
So, event $B_i$ occurs with probability atmost $\delta_i$. Thus, by repeated applications of the union bound, we get:
\begin{align*}
    \Pr[\bar{B}] &\geq 1 - \sum_{i=1}^T \sum_{a_g\in\mathcal{A}_g} \Pr[B_i^{a_g}] \\
    &\geq 1 - |\mathcal{A}_g|\sum_{i=1}^T \Pr[B_i^{a_g}]
\end{align*}
Finally, substituting $\Pr[\bar{B}_i^{a_g}] \leq \delta_i$ yields the lemma. \qedhere \\
\end{proof}

   Recall that for any $s\in\mathcal{S}:=\mathcal{S}_g\times\mathcal{S}_l^n\cong \mathcal{S}_g$, the policy function ${\pi}^\est_{k,m}(s)$ is defined as a uniformly random element in the maximal set of $\hat{\pi}_{k,m}^\est$ evaluated on all possible choices of $\Delta$. Formally:
\begin{equation}
{\pi}_{k,m}^\est(s) \sim\mathcal{U} \left\{\hat{\pi}^\est_{k,m}(s_g,F_{s_\Delta}): \Delta\in \binom{[n]}{k}\right\}
\end{equation}

We now use the celebrated performance difference lemma from \cite{Kakade+Langford:2002}, restated below for convenience in \cref{theorem: performance difference lemma}, to bound the value functions generated between ${\pi}_{k,m}^\est$ and $\pi^*$.\\

\begin{theorem}[Performance Difference Lemma]\label{theorem: performance difference lemma} Given policies $\pi_1, \pi_2$, with corresponding value functions $V^{\pi_1}$, $V^{\pi_2}$:
\begin{align*}
    V^{\pi_1}(s) - V^{\pi_2}(s) &= \frac{1}{1-\gamma} \E_{\substack{s'\sim d^{\pi_1}_s \\ a'\sim \pi_1(\cdot|s')} } [A^{\pi_2}(s',a')]
\end{align*}
Here, $A^{\pi_2}(s',a'):= Q^{\pi_2}(s',a') - V^{\pi_2}(s')$ and $d_s^{\pi_1}(s') = (1-\gamma) \sum_{h=0}^\infty \gamma^h \Pr_h^{\pi_1}[s',s]$ where $\Pr_h^{\pi_1}[s',s]$ is the probability of $\pi_1$ reaching state $s'$ at time step $h$ starting from state $s$.\\
\end{theorem}

\begin{theorem}[Bounding value difference] For any $s\in\mathcal{S}:=\mathcal{S}_g\times\mathcal{S}_l^n$ and $(\delta_1,\delta_2)\in(0,1]^2$, we have: \label{theorem: application of PDL}
\[V^{\pi^*}(s) - V^{{{\pi}}^{\est}_{k,m}}(s) \leq \frac{2\|r_l(\cdot,\cdot)\|_\infty}{(1-\gamma)^2}\sqrt{\frac{n-k+1}{2nk}}\sqrt{\ln \frac{2|\mathcal{S}_l|}{\delta_1}} +  \frac{2\tilde{r}}{(1-\gamma)^2}|\mathcal{A}_g|\delta_1 + \frac{2\epsilon_{k,m}}{1-\gamma}\]  
\end{theorem}
\begin{proof}
Note that by definition of the advantage function, we have:
\begin{align*}
\E_{a'\sim {{\pi}}_{k,m}^{\est}(\cdot|s')} A^{{\pi}^*}(s',a') &= \E_{a'\sim {{\pi}}_{k,m}^{\est}(\cdot|s')}[Q^{{\pi}^*}(s',a') - V^{{\pi}^*}(s')] \\
&= \E_{a'\sim {{\pi}}_{k,m}^\est(\cdot|s')}[Q^{{\pi}^*}(s',a') - \E_{a\sim {\pi}^*(\cdot|s')} Q^{{\pi}^*}(s',a)] \\
&= \E_{a'\sim {{\pi}}_{k,m}^\est(\cdot|s')}\E_{a\sim {\pi}^*(\cdot|s')}[Q^{{\pi}^*}(s',a') - Q^{{\pi}^*}(s',a)]
\end{align*}
Since ${\pi}^*$ is a deterministic policy, we can write:
\begin{align*}
\E_{a'\sim {{\pi}}_{k,m}^\est(\cdot|s')}\E_{a\sim {\pi}^*(\cdot|s')} A^{{\pi}^*}(s',a') &= \E_{a'\sim {{\pi}}_{k,m}^\est(\cdot|s')}[Q^{{\pi}^*}(s',a') - Q^{{\pi}^*}(s',{\pi}^*(s'))] \\
&= \frac{1}{\binom{n}{k}}\sum_{\Delta\in\binom{[n]}{k}}[Q^{{\pi}^*}(s',\hat{\pi}^\est_{k,m}(s_g',F_{s_\Delta'})) - Q^{{\pi}^*}(s',{\pi}^*(s'))]
\end{align*}
\noindent Then, by the linearity of expectations and the performance difference lemma (while noting that $Q^{\pi^*}(\cdot,\cdot) = Q^*(\cdot,\cdot)$):
\begin{align*} V^{{\pi}^*}(s) - V^{{{\pi}}_{k,m}^\est}(s) &= \frac{1}{1-\gamma}\sum_{\Delta\in\binom{[n]}{k}}\frac{1}{\binom{n}{k}}\E_{s'\sim d_s^{{{\pi}}_{k,m}^\est}}\left[Q^{{\pi}^*}(s',{\pi}^*(s')) - Q^{{\pi}^*}(s',{\hat{\pi}}_{k,m}^\est(s_g', F_{s_\Delta'}))\right] \\
&= \frac{1}{1-\gamma}\sum_{\Delta\in\binom{[n]}{k}}\frac{1}{\binom{n}{k}}\E_{s'\sim d_s^{{{\pi}}_{k,m}^\est}}\left[Q^*(s',{\pi}^*(s')) - Q^*(s',{\hat{\pi}}_{k,m}^\est(s'_g,F_{s_\Delta'}))\right]
\end{align*}
Next, we use \cref{lemma: expected_q*bound_with_different_actions} to bound this difference (where the probability distribution function of $\mathcal{D}$ is set as $ d_s^{{\pi}_{k,m}^\est}$ as defined in \cref{theorem: performance difference lemma}) while letting $\delta_1=\delta_2$:
\begin{align*}
&V^{{\pi}^*}(s) - V^{{{\pi}}_{k,m}^\est}(s) \\
&\leq \frac{1}{1-\gamma}\sum_{\Delta\in\binom{[n]}{k}}\frac{1}{\binom{n}{k}}\bigg[\frac{2\|r_l(\cdot,\cdot)\|_\infty}{1-\gamma}\sqrt{\frac{n-k+1}{2nk}}\left(\sqrt{\ln \frac{2|\mathcal{S}_l|}{\delta_1}}\right)+\frac{2\tilde{r}}{1-\gamma}|\mathcal{A}_g|\delta_1 + 2\epsilon_{k,m}\bigg] \\
&\leq \frac{2\|r_l(\cdot,\cdot)\|_\infty}{(1-\gamma)^2}\sqrt{\frac{n-k+1}{2nk}}\left(\sqrt{\ln\frac{2|\mathcal{S}_l|}{\delta_1}}\right) +  \frac{2\tilde{r}}{(1-\gamma)^2}|\mathcal{A}_g|\delta_1+ \frac{2\epsilon_{k,m}}{1-\gamma}
\end{align*}
This proves the theorem. \qedhere\\ \end{proof}

\begin{lemma}\label{lemma: expected_q*bound_with_different_actions}
For any arbitrary distribution $\mathcal{D}$ of states $\mathcal{S} := \mathcal{S}_g\times\mathcal{S}_l^n$, for any $\Delta\in\binom{[n]}{k}$ and for $\delta_1,\delta_2 \in (0,1]$, we have:
\begin{align*}&\E_{s'\sim \mathcal{D}}[Q^*(s',\pi^*(s')) - Q^*(s',\hat{\pi}^{\est}_{k,m}(s_g',F_{s_\Delta')})]\\ &\quad\quad\quad\leq \frac{2\|r_l(\cdot,\cdot)\|_\infty}{1-\gamma}\sqrt{\frac{n-k+1}{8nk}}\left(\sqrt{\ln \frac{2|\mathcal{S}_l|}{\delta_1}} + \sqrt{\ln \frac{2|\mathcal{S}_l|}{\delta_2}}\right) +  \frac{\tilde{r}}{1-\gamma}|\mathcal{A}_g|(\delta_1 + \delta_2) + 2\epsilon_{k,m}
\end{align*}
\end{lemma}
\begin{proof}
Denote $\zeta_{k,m}^{s,\Delta}:=Q^*(s,\pi^*(s)) - Q^*(s, \hat{\pi}_{k,m}^\est(s_g,F_{s_\Delta})$.
We define the indicator function $\mathcal{I}:\mathcal{S}\times \mathbb{N} \times (0,1]\times(0,1]$ by:
\[\mathcal{I}(s, k, \delta_1,\delta_2)=\mathbbm{1}\left\{\zeta_{k,m}^{s,\Delta}\leq\frac{2\|r_l(\cdot,\cdot)\|_\infty}{1-\gamma}\sqrt{\frac{n-k+1}{8nk}}\left(\sqrt{\ln\frac{2|\mathcal{S}_l|}{\delta_1}}+ \sqrt{\ln\frac{2|\mathcal{S}_l|}{\delta_2}}\right)+ 2\epsilon_{k,m}\right\}\]
We then study the expected difference between $Q^*(s',\pi^*(s'))$ and $Q^*(s',\hat{\pi}_{k,m}^{\est}(s'_g,F_{s'_\Delta}))$. Observe that:
\begin{align*}
\E_{s'\sim \mathcal{D}}[\zeta_{k,m}^{s,\Delta}] &= \E_{s'\sim \mathcal{D}}[Q^*(s',\pi^*(s'))\!-\!{Q}^*(s', \hat{\pi}_{k,m}^{\est}(s'_g,F_{s'_\Delta}))] \\
    &= \E_{s'\sim \mathcal{D}}\left[\mathcal{I}(s',k,\delta_1,\delta_2)(Q^*(s',\pi^*(s')) - Q^*(s',\hat{\pi}^{\est}_{k,m}(s_g',F_{s_\Delta'})))\right] \\
    &+ \E_{s'\sim \mathcal{D}}[(1-\mathcal{I}(s',k,\delta_1,\delta_2))(Q^*(s',\pi^*(s')) - Q^*(s',\hat{\pi}^{\est}_{k,m}(s_g',F_{s_\Delta'})))]
\end{align*}

Here, we have used the general property for a random variable $X$ and constant $c$ that $\E[X]=\E[X\mathbbm{1}\{X\leq c\}] + \E[(1-\mathbbm{1}\{X\leq c\})X]$. 
Then,
\begin{align*}
    \E_{s'\sim \mathcal{D}}[Q^*(s',\pi^*(s')) &-{Q}^*(s', \hat{\pi}_{k,m}^{\est}(s_g',F_{s_\Delta'})] \\
    &\leq \frac{2\|r_l(\cdot,\cdot)\|_\infty}{1-\gamma}\sqrt{\frac{n-k+1}{8nk}}\left(\sqrt{\ln\frac{2|\mathcal{S}_l|}{\delta_1}} + \sqrt{\ln\frac{2|\mathcal{S}_l|}{\delta_2)}}\right) + 2\epsilon_{k,m}\\
    &\quad\quad\quad\quad\quad\quad+  \frac{\tilde{r}}{1-\gamma}\left(1-\E_{s'\sim\mathcal{D}} \mathcal{I}(s',k,\delta_1,\delta_2)\right))\\
    &\leq \frac{2\|r_l(\cdot,\cdot)\|_\infty}{1-\gamma}\sqrt{\frac{n-k+1}{8nk}}\left(\sqrt{\ln\frac{2|\mathcal{S}_l|}{\delta_1}} + \sqrt{\ln\frac{2|\mathcal{S}_l|}{\delta_2)}}\right) + 2\epsilon_{k,m}\\
    &\quad\quad\quad\quad\quad\quad +  \frac{\tilde{r}}{1-\gamma}|\mathcal{A}_g|(\delta_1 + \delta_2) 
\end{align*}For the first term in the first inequality, we use  $\E[X\mathbbm{1}\{X\leq c\}] \leq c$. For the second term, we trivially bound $Q^*(s',\pi^*(s')) - Q^*(s',\hat{\pi}_{k,m}^{\est}(s_g',F_{s_\Delta'}))$ by the maximum value $Q^*$ can take, which is $\frac{\tilde{r}}{1-\gamma}$ by \cref{lemma: Q-bound}. In the second inequality, we use the fact that the expectation of an indicator function is the conditional probability of the underlying event. The second inequality follows from \cref{lemma: q_star_different_action_bounds} which yields the claim.\qedhere \\
\end{proof}

\begin{lemma}\label{lemma: q_star_different_action_bounds} For a fixed $s'\in\mathcal{S}:=\mathcal{S}_g\times\mathcal{S}_l^n$, for any $\Delta\in\binom{[n]}{k}$, and for $\delta_1,\delta_2\in (0,1]$, we have that with probability at least $1 - |\mathcal{A}_g|(\delta_1 + \delta_2)$:
    \[Q^*(s',\pi^*(s')) - Q^*(s',\hat{\pi}_{k,m}^{\est}(s_g',F_{s_\Delta'})) \!\leq\! \frac{2\|r_l(\cdot,\cdot)\|_\infty}{1-\gamma}\!\sqrt{\frac{n-k+1}{8nk}}\!\!\left(\!\!\!\sqrt{\ln\frac{2|\mathcal{S}_l|}{\delta_1}}\!+\! \sqrt{\ln\frac{2|\mathcal{S}_l|}{\delta_2}}\right)\!+ 2\epsilon_{k,m}\]
\end{lemma}
\begin{proof}
\begin{align*}
    Q^*(s',\pi^*(s'))&-Q^*(s',\hat{\pi}_{k,m}^{\est}(s_g',F_{s_\Delta'})) \\
    &= Q^*(s',\pi^*(s'))\!-\!Q^*(s',\hat{\pi}_{k,m}^{\est}(s_g',F_{s_\Delta')})\!+\! \hat{Q}^{\est}_{k,m}(s_g',s_\Delta',\pi^*(s')) \\
    &-\hat{Q}^{\est}_{k,m}(s_g',s_\Delta',\pi^*(s'))\!+\!\hat{Q}_{k,m}^{\est}(s_g',s_\Delta',\hat{\pi}_{k,m}^{\est}(s_g',F_{s_\Delta'}))\!\\
    &\quad\quad -\!\hat{Q}_{k,m}^{\est}(s_g',F_{s_\Delta'},\hat{\pi}_{k,m}^{\est}(s_g',F_{s_\Delta'}))
\end{align*}
By the monotonicity of the absolute value and the triangle inequality, we have:
\begin{align*}
Q^*(s',\pi^*(s'))&-Q^*(s',\hat{\pi}_{k,m}^{\est}(s_g',F_{s_\Delta'})) \\
&\leq |Q^*(s',\pi^*(s'))-\hat{Q}_{k,m}^{\est}(s_g',F_{s_\Delta'},\pi^*(s'))| \\
&+ |\hat{Q}_{k,m}^{\est}(s_g',F_{s_\Delta'},\hat{\pi}_{k,m}^{\est}(s_g',F_{s_\Delta'}))-Q^*(s',\hat{\pi}_{k,m}^{\est}(s_g',F_{s_\Delta'}))|
\end{align*}
The above inequality crucially uses the fact that the residual term $\hat{Q}_{k,m}^{\est}(s_g',F_{s_\Delta'},\pi^*(s')) - \hat{Q}_{k,m}^{\est}(s_g',F_{s_\Delta'},\hat{\pi}_{k,m}^{\est}(s_g',F_{s_\Delta'})) \leq 0$, since $\hat{\pi}^{\est}_{k,m}$ is the optimal greedy policy for $\hat{Q}_{k,m}^{\est}$. Finally, applying the error bound derived in \cref{lemma:union_bound_over_finite_time} for two timesteps completes the proof. \qedhere \\
\end{proof}

\begin{corollary}Optimizing parameters in \cref{theorem: application of PDL} yields:
\label{corollary: performance_difference_lemma_applied}
\[V^{\pi^*}(s) - V^{{{\pi}}^\est_{k,m}}(s) \leq \frac{2\tilde{r}}{(1-\gamma)^2}\left(\sqrt{\frac{n-k+1}{2nk} \ln(2|\mathcal{S}_l||\mathcal{A}_g|\sqrt{k})} +  \frac{1}{\sqrt{k}}\right) + \frac{2\epsilon_{k,m}}{1-\gamma}\]
\end{corollary}
\begin{proof} Recall from \cref{theorem: application of PDL} that:
    \[V^{\pi^*}(s) - V^{{{\pi}}^\est_{k,m}}(s) \leq \frac{2\|r_l(\cdot,\cdot)\|_\infty}{(1-\gamma)^2}\sqrt{\frac{n-k+1}{2nk}}\left(\sqrt{\ln \frac{2|\mathcal{S}_l|}{\delta_1}}\right) +  \frac{2\|r_l(\cdot,\cdot)\|_\infty}{(1-\gamma)^2}|\mathcal{A}_g|\delta_1+ \frac{2\epsilon_{k,m}}{1-\gamma}\] 
    Note $\|r_l(\cdot,\cdot)\|_\infty \leq \tilde{r}$ from Assumption \ref{assumption: bounded rewards}. Then,
    \[V^{\pi^*}(s) - V^{{{\pi}}^\est_{k,m}}(s) \leq \frac{2\tilde{r}}{(1-\gamma)^2}\left(\sqrt{\frac{n-k+1}{2nk} \ln\frac{2|\mathcal{S}_l|}{\delta_1}} +  |\mathcal{A}_g|\delta_1\right) + \frac{2\epsilon_{k,m}}{1-\gamma}\]
Finally, setting $\delta_1 = \frac{1}{k^{1/2}|\mathcal{A}_g|}$ yields the claim.\qedhere \\
\end{proof}

\begin{corollary}\label{pdl result in otilde form}
    Therefore, from \cref{corollary: performance_difference_lemma_applied}, we have:
    \begin{align*}V^{\pi^*}(s) - V^{{{\pi}}^\est_{k,m}}(s) &\leq {O}\left(\frac{\tilde{r}}{\sqrt{k}(1-\gamma)^2}\sqrt{ \ln(2|\mathcal{S}_l||\mathcal{A}_g|\sqrt{k})} + \frac{\epsilon_{k,m}}{1-\gamma}\right) \\
    &= \widetilde{O}\left(\frac{\tilde{r}(1-\gamma)^{-2}}{\sqrt{k}} + \frac{\epsilon_{k,m}}{1-\gamma}\right)\end{align*}
    This yields the bound from \cref{theorem: performance_difference_lemma_applied}.
\end{corollary}

\section{Additional Discussions}

\begin{discussion}[Tighter Endpoint Analysis]
\emph{Our theoretical result shows that $V^{\pi^*}(s) - V^{\pi_{k,m}^\est}$ decays on the order of $O(1/\sqrt{k}+\epsilon_{k,m})$. For $k=n$, this bound is actually suboptimal since $\hat{Q}_k^*$ becomes $Q^*$. However, placing $|\Delta|=n$ in our weaker TV bound in Lemma \ref{lemma: tv_distance_bretagnolle_huber}, we recovers a total variation distance of $0$ when $k=n$, recovering the optimal endpoint bound.}
\end{discussion}

\begin{discussion}[Choice of $k$] \emph{Discussion \ref{disussion: complexity requirement} previously discussed the tradeoff in $k$ between the polynomial in $k$ complexity of learning the $\hat{Q}_k$ function and the decay in the optimality gap of $O(1/\sqrt{k})$. This discussion promoted $k=O(\log n)$ as a means to balance the tradeoff. However, the ``correct'' choice of $k$ truly depends on the amount of compute available, as well as the accuracy desired from the method. If the former is available, we recommend setting $k= \Omega(n)$ as it will yield a more optimal policy. Conversely, setting $k=O(\log n)$, when $n$ is large, would be the minimum $k$ recommended to realize any asymptotic decay of the optimality gap.}
\end{discussion}
\end{document}